\documentclass[twoside]{article}
\usepackage[accepted]{aistats2020}
\usepackage[round]{natbib}

\usepackage{hyperref}
\usepackage{graphicx}
\usepackage[dvipsnames]{xcolor}
\usepackage{macros}
\usepackage{algorithm}
\usepackage{algorithmic}
\usepackage{booktabs}
\usepackage{lipsum}
\usepackage{mathabx}
\usepackage{tabu}
\usepackage{longtable}
\usepackage[toc,page,header]{appendix}

\def\rlsvi{\textsc{RLSVI} }
\def\pk{{k}}
\def\rtsa{r_t(s,a)}

\def\Regret{\textsc{Regret}}

\def\thetahat{\widehat \theta}

\def\vo{\overline V}

\def\H{\mathcal H}
\def\G{\mathcal G}

\def\Ok{\mathcal O_k}

\def\thetabar{\overline\theta}

\def\stk{s_{tk}}
\def\sok{s_{1k}}

\def\Vbar{\overline V}

\newcommand{\qo}[1]{\overline Q_{#1}}

\newcommand{\deltacorr}[3]{(\Delta_{#1}^R(#2,#3) +{\Delta_{#1}^P(#2,#3)}^\top\vo_{#1+1,k}) }
\newcommand{\deltacorrsa}[1]{\deltacorr{t}{s}{a}}

\def\thetabar {\overline \theta}

\def\phitsat{\phi_t(s,a)^\top}

\newcommand{\Smallcondition}[2]{\|#2 \|_{\Sigma^{-1}_{#1}} \leq \alpha_L }
\newcommand{\NonSmallcondition}[2]{\|#2 \|_{\Sigma^{-1}_{#1}} > \alpha_L }
\newcommand{\Condition}[2]{\|#2\|_{\Sigma^{-1}_{#1}} }
\newcommand{\Mediumcondition}[2]{\alpha_L < \|#2 \|_{\Sigma^{-1}_{#1}} < \alpha_U }
\newcommand{\Largecondition}[2]{\|#2 \|_{\Sigma^{-1}_{#1}} \geq \alpha_U }

\newcommand{\sclip}[1]{\textsc{Clip}_{[-3(H-t),3(H-t)]}}

\newcommand{\psmall}[1]{\1\{x_t \text{is small}\}}
\newcommand{\plarge}[1]{\1\{x_t \text{is large}\}}

\newtheorem{lem}{Lemma}[section]

\AtAppendix{\counterwithin{lem}{section}}

\usepackage{longtable}

\usepackage[toc,page,header]{appendix}
\usepackage{minitoc}

\let\Contentsline\contentsline 
\renewcommand
\makeatletter
\renewcommand{\@dotsep}{10000} 
\makeatother

\doparttoc 
\faketableofcontents 

\runningtitle{Frequentist Regret Bounds for Randomized Least-Squares Value Iteration}

\runningauthor{Zanette, Brandfonbrener, Brunskill, Pirotta, Lazaric}

\begin{document}
\twocolumn[

\aistatstitle{Frequentist Regret Bounds for \\  Randomized Least-Squares Value Iteration}

\aistatsauthor{Andrea Zanette* \And David Brandfonbrener*}

\aistatsaddress{Stanford University \And  New York University }

\aistatsauthor{Emma Brunskill \And Matteo Pirotta \And Alessandro Lazaric}

\aistatsaddress{Stanford University \And Facebook AI Research \And Facebook AI Research} ]

\let\thefootnote\relax\footnotetext{*Equal Contribution}

\begin{abstract}
We consider the exploration-exploitation dilemma in finite-horizon reinforcement learning (RL). When the state space is large or continuous, traditional tabular approaches 
are unfeasible and some form of function approximation is mandatory. In this paper, we introduce an optimistically-initialized variant of the popular randomized least-squares value iteration (RLSVI), a model-free algorithm where exploration is induced by perturbing the least-squares approximation of the action-value function. Under the assumption that the Markov decision process has low-rank transition dynamics, we prove that the frequentist regret of RLSVI is upper-bounded by $\widetilde O(d^2 H^2 \sqrt{T})$ where $ d $ is the feature dimension, $ H $ is the horizon, and $ T $ is the total number of steps. To the best of our knowledge, this is the first frequentist regret analysis for randomized exploration with function approximation.

\end{abstract}

\section{Introduction}
A key challenge in reinforcement learning (RL) is how to balance exploration and exploitation in order to efficiently learn to make good sequences of decisions in a way that is both computationally tractable and statistically efficient.  
In the tabular case, the exploration-exploitation problem is well-understood for a number of settings (e.g., finite-horizon, average reward, infinite horizon with discount), exploration objectives (e.g., regret minimization and probably approximately correct), and for different algorithmic approaches, where optimism-under-uncertainty~\citep{Jaksch10,Fruit18} and Thompson sampling (TS)~\citep{osband2016deep,russo2019worst} are the most popular principles. For instance, in the finite-horizon setting, \citet{Azar17} and \citet{zanette2019tighter} recently derived minimax optimal and structure adaptive regret bounds for optimistic exploration algorithms.  TS-based algorithms have mainly been analyzed in tabular MDPs in terms of Bayesian regret~\citep{osband2016deep,OV17,OuyangGN017}, which assumes that the MDP is sampled from a known prior distribution. These bounds do not hold against a fixed MDP and algorithms with small Bayesian regret may still suffer high regret in some hard-to-learn MDPs within the chosen prior. In the tabular setting, frequentist (or worst-case) regret analysis has been developed for TS-based algorithms both in the average reward~\citep{GopalanM15,Agrawal2017} and finite-horizon case~\citep{russo2019worst}. Despite the fact that TS-based approaches have slightly worse regret bounds compared to optimism-based algorithms, their empirical performance is often superior~\citep{Chapelle2011empiricalTS,OV17}. 

Unfortunately, the performance of tabular exploration methods rapidly degrades with the number of states and actions, thus making them infeasible in large or continuous MDPs. So, one of the most important challenges to improve sample efficiency in large-scale RL is how to combine exploration mechanisms with generalization methods to obtain algorithms with provable regret guarantees. The simplest approach to deal with continuous state is discretization.
It has been used in \citet{ortner2012continuous,Lakshmanan2015improvedcont} to derive $\widetilde{O}(T^{3/4})$ and $\widetilde{O}(T^{2/3})$ frequentistic regret bounds for average reward MDPs. Recent work on contextual MDPs~\citep{jiang17contextual,dann2018oracle} yielded promising sample efficiency guarantees, but such algorithms are computationally intractable, and their bounds are not tight in the tabular settings. 

One of the most simple and popular forms of function approximation is to use a linear representation for the action-value functions. When the transition model also has low-rank structure, very recent work has shown that a variant of $Q$-learning can achieve polynomial sample complexity as a function of the state space dimension when given access to a generative model~\citep{yang2019sample}. Nonetheless, the generative model assumption removes most of the exploration challenge, as the state space can be arbitrarily sampled. Concurrently to our work, optimism-based exploration has been successfully integrated with linear function approximation both in model-based and model-free algorithms~\citep{yang2019reinforcement,jin2019provably}. In MDPs with low-rank dynamics, these algorithms are proved to have regret bounds scaling with the dimensionality $d$ of the linear space (i.e., the number of features) instead of the number of states.

On the algorithmic side, TS-based exploration can be easily integrated with linear function approximation as suggested in the Randomized Least-Squares Value Iteration (RLSVI) algorithm~\citep{osband2016generalization}. Despite promising empirical results, RLSVI has been analyzed only in the tabular case (i.e., when the features are indicators for each state) and for Bayesian regret. While RLSVI is a model-free algorithm, recent work~\citep{russo2019worst} leverages an equivalence between model-free and model-based algorithms in the tabular case to derive frequentist regret bounds. The analysis carefully chooses the variance of the perturbations applied to the estimated solution to ensure that the value estimates are optimistic with constant probability. 

In this paper we provide the first frequentist regret analysis for a variant of RLSVI when linear function approximation is used in the finite-horizon setting. Similar to optimistic PSRL for the tabular setting~\citep{Agrawal2017}, we modify RLSVI to ensure that the perturbed estimates used in the value iteration process are optimistic with constant probability. Following the results in the linear bandit literature~\citep{abeille2017linear}, we show that the perturbation applied to the the least-squares estimates should be larger than their estimation error. However, in contrast to bandit, perturbed estimates are propagated back through iterations and we need to carefully adjust the perturbation scheme so that the probability of being optimistic does not decay too fast with the horizon and, at the same time, we can control how the perturbations accumulate over iterations. Under the assumption that the system dynamics are low-rank, we show that the frequentist regret of our algorithm is $\widetilde O(H^2 d^2 \sqrt{T} + H^5 d^4  + \epsilon dH (1 + \epsilon dH^2)T)$ where $\epsilon$ is the misspecification level, $H$ is the fixed horizon, $d$ is the number of features, and $T$ is the number of samples. Similar to linear bandits, this is worse by a factor of $\sqrt{Hd}$ (i.e., the square root of the dimension of the estimated parameters) than the optimistic algorithm of \citet{jin2019provably}. Whether this gap can be closed is an open question both in bandits and RL.

\section{Preliminaries}
\label{rlsvi:sec:Preliminaries}
We consider an undiscounted finite-horizon MDP~\citep{puterman1994markov} $M = (\mathcal{S}, \mathcal{A}, \mathbb{P}, r, H)$ with state space $\mathcal{S}$, action space $\mathcal{A}$ and horizon length $H \in\mathbb{N}^+$.
For every $t \in [H] \defeq \{1, \ldots, H\}$, every state-action pair is characterized by a reward $r_t(s,a) \in [0,1]$ and a transition kernel $\mathbb{P}_t(\cdot|s,a)$ over next state.
We assume $\mathcal{S}$ to be a measurable, possibly infinite, space and $\mathcal{A}$ can be any (compact) time and state dependent set (we omit this dependency for brevity).
For any $t \in [H]$ and $(s,a) \in \mathcal{S} \times \mathcal{A}$, the state-action value function of a non-stationary policy $\pi = (\pi_1, \ldots, \pi_H)$ is defined as $ Q^{\pi}_{t}(s,a) = r_{t}(s,a) + \mathbb{E} \left[ \sum_{l = t+1}^{H} r_{l}(s_{l}, \pi_{l}(s_l))\mid s,a \right]$
and the value function is $V^{\pi}_t(s) = Q^{\pi}_t(s, \pi_t(s))$.
Since the horizon is finite, under some regularity conditions, \citep{shreve1978alternative}, there always exists an optimal policy $\pi^\star$ whose value and action-value functions are defined as $\Vstar_t(s) \defeq V^{\pi^\star}_t(s) = \sup_{\pi} V^{\pi}_t(s)$ and $\Qstar_t(s,a) \defeq Q^{\pi^\star}_t(s,a) = \sup_{\pi} Q^{\pi}_t(s,a)$. Both $Q^\pi$ and $Q^\star$ can be conveniently written as the result of the Bellman equations
\begin{align}
Q^{\pi}_t(s,a) &= r_t(s,a) + \mathbb{E}_{s' \sim \mathbb{P}_t(\cdot|s,a)}[V^{\pi}_{t+1}(s')]\\
Q^{\star}_t(s,a) &= r_t(s,a) + \mathbb{E}_{s' \sim \mathbb{P}_t(\cdot|s,a)}[V^{\star}_{t+1}(s')]
\end{align}
where $V^{\pi}_{H+1}(s) =V^{\star}_{H+1}(s) = 0$ and $V^\star_t(s) = \max_{a\in\mathcal{A}} Q^\star_t(s,a)$, for all $s \in \mathcal{S}$. Notice that by boundedness of the reward, for any $t$ and $(s,a)$, all functions $Q^\pi_t$, $V^\pi_t$, $Q^\star_t$, $V^\star_t$ are bounded in $[0, H-t+1]$. 

\paragraph{The learning problem}
The learning agent interacts with the MDP in a sequence of episodes $k \in [K]$ of fixed length $H$ by playing a nonstationary policy $\pi_k = (\pi_{1k}, \ldots, \pi_{Hk})$ where $\pi_{tk} : \mathcal{S} \to \mathcal{A}$.
In each episode, the initial state $s_{1k}$ is chosen arbitrarily and revealed to the agent. 
The learning agent does not know the transition or reward functions, and it relies on the samples (i.e., states and rewards) observed over episodes to improve its performance over time. Finally, we evaluate the performance of an agent by its regret after $K$ episodes: $\textsc{Regret(K)} \defeq \sum_{k=1}^K V^\star_1(s_{1k}) - V^{\pi_k}_1(s_{1k})$. 

\paragraph{Linear function approximation and low-rank MDPs.} 
Whenever the state space $\mathcal{S}$ is too large or continuous, functions above cannot be represented by enumerating their values at each state or state-action pair. A common approach is to define a feature map $\phi_t: \StateSpace  \times\ActionSpace \rightarrow \R^d$, possibly different at any $ t \in [H]$, embedding each state-action pair $(s,a)$ into a $d$-dimensional vector $\phi_t(s,a)$. The action-value functions are then represented as a linear combination between the features $\phi_t$ and a vector parameter $\theta_t\in\R^d$, such that $Q_t(s,a) = \phi_t(s,a)^\top \theta_t$. This representation effectively reduces the complexity of the problem from $\StateSpace \times \ActionSpace$ down to $d$. Nonetheless, $Q^\star_t$ may not fit into the space spanned by $\phi_t$, and approximate value iteration may propagate and accumulate errors over iterations \citep{munos2005error,munos2008finite}, and an exploration algorithm may suffer linear regret. Thus, similar to~\citep{yang2019reinforcement,yang2019sample,jin2019provably}, we consider MPDs that are ``coherent'' with the feature map $\phi_t$ used to represent action-value functions. In particular, we assume that $M$ has (approximately) low-rank transition dynamics and linear reward in $\phi_t$. 

\begin{assumption}[Approximately Low-Rank MDPs]
\label{rlsvi:assMainAssumption}
We assume that for each $ t \in [H] $ there exist a feature map $\psi_t:  \StateSpace   \rightarrow \R^d, \; s \mapsto \psi_t(s)$
and a parameter $\theta^r_t \in \R^d$ such that
the reward can be decomposed as a linear response and a non-linear term:
\begin{align}
\rtsa = \phitsat \theta^r_t + \Delta^r_t(s,a)
\end{align}
and the dynamics are approximately low-rank:
\begin{align}
\label{eqn:LowRankTransition}
\mathbb{P}_t(s' \mid s,a) = \phitsat \psi_t(s') + \Delta_t^P(s'\mid s,a).
\end{align}
We denote by $\epsilon$ an upper bound on the non-linear terms, as follows:
\begin{align}
|\Delta^r_t(s,a) | & \leq \epsilon, \quad \quad \|\Delta^P_t(\cdot \mid s,a) \|_{1} \leq \epsilon.
\end{align}
We further make the following regularity assumptions:
\begin{align}
\label{eqn:regularity}
	\| \phi_t(s,a)\|_2  \leq L_\phi, \; \|\theta^r_t\|_2  \leq L_r, \; \int_{s}\|\psi_t(s)\| \leq L_\psi.
\end{align}
\end{assumption}
An important consequence of Asm.~\ref{rlsvi:assMainAssumption} in the absence of misspecification ($\epsilon = 0)$ is that the Q-function of any policy is linear in the features $ \phi$.
\begin{proposition}
	\label{rlsvi:proplinearQ_maintext}
    If $\epsilon=0$, for every policy $\pi$ and timestep $t \in [H]$ there exists $\theta^\pi_t \in \mathbb{R}^d$ such that
    \begin{align}
    	Q^\pi_t(s,a) = \phi_t(s,a)^\top \theta_t^\pi, \quad \forall (s,a) \in \StateSpace \times \ActionSpace.
    \end{align}
\end{proposition}
\begin{proof}
The definition of low-rank MDP from Asm.~\ref{rlsvi:assMainAssumption} together with the Bellman equation gives:
\begin{align*}
    Q_t^\pi(s,a) &= r_t(s,a) + \E_{s'|s,a}[V_{t+1}^\pi(s')]\\
    &= \phi_t(s,a)^\top \theta_t^r + \int_{s'}\phi_t(s,a)^\top \psi_t(s') V_{t+1}^\pi(s')\\
    &= \phi_t(s,a)^\top \(\theta_t^r + \int_{s'}\psi_t(s') V_{t+1}^\pi(s')\) \numberthis{\label{eqn:linearQ}}
\end{align*}
We define $ \theta_t^\pi $ to be the term inside the parentheses.
\end{proof}
To give further intuition about the assumption, consider the case of finite state and action spaces (again with $ \epsilon = 0$). Then we can write:
\begin{align}
	\mathbb{P}_t(s,a) = \phi_t(s,a)^\top \Psi_t
\end{align}
for a certain $\Psi_t \in \R^{d\times\StateSpace}$. Then for any policy $\pi$ there exists a matrix $\Phi^\pi$ such that the transition matrix of the Markov chain $P^\pi$ can be expressed by a low-rank factorization:
\begin{align}
\label{eqn:LowRankFactorization}
P_t^\pi = \Phi_t^\pi\Psi_t, \quad \Phi_t^\pi \in \R^{S\times d}, \Psi_t \in \R^{d\times S}
\end{align}
where in particular $\Phi^\pi_t$ depends on the policy $\pi$:
\begin{align}
\Phi^\pi_t[s,:] = \phi_t(s,\pi(s))^\top, \quad \Psi_t[:,s'] = \psi_t(s').
\end{align}
Since $\textsc{Rank}(\Phi^\pi_t) \leq d, \textsc{Rank}(\Psi_t)\leq d$ we get $\textsc{Rank}(P^\pi_t) \leq d$ (see \citet{golub2012matrix}).
\section{Algorithm}

Our primary goal in this work is to provide a Thompson sampling (TS)-based algorithm with linear value function approximation with frequentist regret bounds. 
A key challenge in frequentist analyses of TS algorithms is to ensure sufficient exploration using randomized (i.e., perturbed) versions of the estimated model or value function. A common way to obtain effective exploration has been to consider perturbations large enough so that the resulting sampled model or value function is optimistic with a fixed probability~\citep{AgrawalG13, abeille2017linear, russo2019worst}. However, such prior work has only considered the bandit or tabular MDP settings. Here we modify \rlsvi described by
\citet{osband2016generalization} to use an optimistic ``default'' value function during an initial phase and inject carefully-tuned perturbations to enable \emph{frequentist} regret bounds in low-rank MDPs. We refer to the resulting algorithm as \textsc{opt-rlsvi} and we illustrate it in Alg.~\ref{alg:RLSVI}.

\textbf{Gaussian noise to encourage exploration.} 
\textsc{opt-rlsvi} proceeds in episodes. At the beginning of episode $k$ it receives an initial state $s_{1k}$ and runs a value iteration procedure to compute a linear approximation of $Q^\star_t$ at each timestep $t \in [H]$. 
To encourage exploration, the learned parameter $\widehat \theta_{tk}$ is perturbed by adding mean-zero Gaussian noise $\overline \xi_{tk} \sim \mathcal N(0, \sigma^2\Sigma^{-1}_{tk})$, obtaining $\overline{\theta}_{tk} = \widehat{\theta}_{tk} + \overline{\xi}_{tk}$. The perturbation (or \textit{pseudonoise}) $\overline \xi_{tk}$ has variance proportional to the inverse of the regularized design matrix $\Sigma_{tk} = \sum_{i=1}^{k-1} \phi_{ti}\phi_{ti}^\top + \lambda I$, where the $\phi_{ti}$'s are the features encountered in prior episodes; this results in perturbations with higher variance in less explored directions. Finally, we show how to choose the magnitude $ \sigma^2 $ of the variance in Sec.~\ref{euler:sec:NoiseBoundMain} to ensure sufficient exploration.

A key contribution of our work is to prove that this strategy can guarantee reliable exploration under Asm.~\ref{rlsvi:assMainAssumption}. We do this by showing that the algorithm is optimistic with constant probability. Explicitly, we prove that the (random) value function difference $(\Vbar_{1k}-\Vstar_1)(s_{1k})$ can be expressed as a \emph{one-dimensional biased random walk}, which depends on a high probability bound on the environment noise (the bias of the walk) and on the variance of the injected pseudonoise (the variance of the walk). By setting the pseudonoise to have the appropriate variance we can guarantee that the random walk is ``optimistic'' enough that the algorithm explores sufficiently. Unfortunately, it is possible to analyze the algorithm as a random walk only if the value function is not perturbed by clipping; otherwise, one cannot write down the walk and the process is difficult to analyze as further bias is introduced by clipping. However, not clipping the value function may give rise to abnormal values.

\textbf{The issue of abnormal values.} 
A common problem that arises in estimation in RL with function approximation is that as a result of statistical errors combined with the bootstrapping and extrapolation of the next-state value function ~\citep{munos2005error,munos2008finite,farahmand2010error} the value function estimate can take values outside its plausible range.  
A common solution is to ``clip'' the bootstrapped value function into the range of plausible values (in this case, between $0$ and $H$).
This avoids propagating overly abnormal values to the estimated parameters at prior timesteps which would degrade their estimation accuracy.
Clipping the value function is also a solution typically employed in tabular algorithms for exploration \citep{Azar17,Dann17,zanette2019tighter,yang2019reinforcement,dann2019policy}. After adding optimistic bonuses for exploration they ``clip'' the value function above by $H$, which is an upper bound on the true optimal value function. Since $H$ is guaranteed to be an optimistic estimate for $\Vstar$, clipping effectively preserves optimism while keeping the value function bounded for bootstrapping.
However, clipping cannot be easily integrated in our setting as it effectively introduces bias in the pseudonoise and it may ``pessimistically'' affect the value function estimates, reducing the probability of being optimistic. 

\textbf{Default value function.}
To avoid propagating unreasonable values without using clipping, we define a default value function, similar in the spirit to algorithms such as $R_{\max}$~\citep{BrafmanT02}. 
In particular, we assign the maximum plausible value $\overline Q_t(s,a) = H-t+1$ to an uncertain direction $\phi_t(s,a)$  (as measured by the $\|\phi_t(s,a) \|_{\Sigma^{-1}_{tk}}$ norm).
Once a given direction $\phi_t(s,a)$ has been tried a sufficient number of times we can guarantee (under an inductive argument) that the linearity of the representation is accurate enough that with high probability $\phi_t(s,a)^\top\overline \theta_{tk} - \Qstar_t(s,a) \in [-(H-t+1),2(H-t+1)]$. In other words, abnormal values are not going to be encountered, and thus clipping becomes unnecessary. Notice that this accuracy requirement is quite minimal because $\Vstar_t$ has a range of at most $H-t+1$. 

We emphasize that the purpose of the optimistic default function is not to inject further optimism but rather to keep the propagation of the errors under control while ensuring	 optimism.

\textbf{Defining the $\overline Q$ values.}
Finally, we also choose our Q function to interpolate between the ``default'' optimistic value and the linear function of the features as the uncertainty decreases. The main reason is to ensure \emph{continuity} of the function, which facilitates the handling of some of the technical aspects connected to the concentration inequality (in particular in App.~\ref{rlsvi:sec:concentration}).
\begin{definition}[Algorithm Q function]\label{rlsvi:defQdef} For some constants $ \alpha_L, \alpha_U $ and using shorthand for the feature $\phi \defeq \phi_t(s,a)$, the default function $ B_t \defeq H-t + 1$ and the interpolation parameter $\rho \defeq \frac{\Condition{tk}{\phi}- \alpha_L}{\alpha_U-\alpha_L}$ define:
\begin{align*}\label{eqn:defqbar}
\Qbar_{tk}(s,a) &\defeq
\begin{cases}
      \phi^\top \thetabar_{tk}, &\text{if } \Smallcondition{tk}{\phi} \\
      B_t,  &\text{if } \Largecondition{tk}{\phi}\\
      \rho\(\phi^\top \thetabar_{tk}\) + (1-\rho)B_t,  & \text{otherwise.}
    \end{cases}
\end{align*}
\end{definition}

\begin{algorithm}[htb]
   \caption{\textsc{opt-rlsvi}}
   \label{alg:RLSVI}
\begin{algorithmic}[1]
\STATE Initialize $\Sigma_{t1} = \lambda I, \; \forall t \in [H]$; Define $\Vbar_{tk}(s) = \max_a \Qbar_{tk}(s,a)$, with $\Qbar_{tk}(s,a)$ defined in Def.~\ref{rlsvi:defQdef}
\FOR{$k = 1,2,\dots$}
\STATE Receive starting state $s_{1k}$
\STATE Set $ \overline{\theta}_{H+1,k} = 0$
\FOR {$t = H,H-1,\dots,1$}
\STATE $\widehat \theta_{tk} = \Sigma_{tk}^{-1}\(\sum_{i=1}^{k-1} \phi_{ti} [r_{ti} + \overline V_{t+1,k}(s_{t+1,i})]\) $
\STATE Sample $\overline \xi_{tk} \sim \mathcal N(0, \sigma^2 \Sigma^{-1}_{tk} )$ 
\STATE $\overline \theta_{tk} = \widehat \theta_{tk} + \overline \xi_{tk}$
\ENDFOR
\STATE Execute $ \overline \pi_{tk}(s) =  \argmax_{a} \overline Q_{tk}(s,a)$, see Def.~\ref{rlsvi:defQdef}
\STATE Collect trajectories of $ (s_{tk}, a_{tk}, r_{tk})$ for $ t\in [H]$.
\STATE Update  $  \Sigma_{t,k+1} = \Sigma_{tk} + \phi_{tk} \phi_{tk}^\top $ for $ t \in [H]$
\ENDFOR
\end{algorithmic}
\end{algorithm}

\section{Main Result}
We present the first frequentist regret bound for a TS-based algorithm in MDPs with approximate linear reward response and low-rank transition dynamics: 
\begin{theorem}
\label{main:MainResult}
Fix any $ 0 < \delta < \Phi(-1)$ and total number of episodes $ K$. Define $ \delta' = \delta / (16HK)$. Assume Asm.~\ref{rlsvi:assMainAssumption} and set the algorithm parameters $ \lambda = 1$, $\sigma = \sqrt{H\nu_k(\delta')} = \sqrt{H}(\widetilde O(Hd)  + L_\phi ( 3HL_\psi + L_r ) + 4 \epsilon H \sqrt{dk})$,
$ \alpha_U = 1/\widetilde O(\sigma \sqrt{d})$, and $ \alpha_L = \alpha_U/2$ (full definitions with the log terms can be found in App. \ref{rlsvi:sec:good}). Then with probability at least $1-\delta$ the regret of \textsc{opt-rlsvi} up to episode $K$ by:
\begin{align}
\widetilde O\(\sigma dH\sqrt{K}  + \frac{H^2d}{\alpha_L^2} + \epsilon H^2K \).
\end{align}
If we further assume that $ L_\phi = \widetilde O(1)$ and $ L_r, L_\psi = \widetilde O(d)$, then the bound reduces to
\begin{align}\label{eqn:simple_regret}
    \widetilde O\(H^2 d^2 \sqrt{T} + H^5 d^4  + \epsilon dH (1 + \epsilon dH^2) T  \).
\end{align}
\end{theorem}

For the setting of low-rank MDPs a lower bound is currently missing both in terms of statistical rate and regarding the misspecification. Recently, \citet{du2019good,lattimore2019learning,van2019comments} discuss what's possible to achieve regarding the misspecification level while \citet{zanette2020learning} provide a regret lower bound for a setting more general than ours.

For finite action spaces \textsc{opt-rlsvi} can be implemented efficiently in space $O(d^2H+dAHK)$ and time $O(d^2AHK^2)$ where $A$ is the number of actions (Prop.~\ref{rlsvi:propComputationalComplexity} in appendix).

It is useful to compare our result with \citet{yang2019reinforcement} and \citet{jin2019provably} which study a similar setting but with an approach based on deterministic optimism, and with \citet{russo2019worst} which proves worst-case regret bounds of \textsc{Rlsvi} for tabular representations.

\textbf{Comparison with \citet{yang2019reinforcement}.}
Recently, \citet{yang2019reinforcement} studied exploration in finite state-spaces and low-rank transitions. They define a model-based algorithm that tries to learn the ``core matrix'', defined as the middle factor of a three-factor low-rank factorization. While their regularity assumptions on the parameters do not immediately fit in our framework, an important distinction (beyond model-based vs model-free) is that their algorithm potentially needs to compute the value function across all states. This suffers $\Omega(S)$ computational complexity and cannot directly handle continuous state spaces.

\textbf{Comparison with \citet{jin2019provably}.}
A more direct comparison can be done with \citet{jin2019provably} which is based on least-square value iteration (like \textsc{opt-rlsvi}) and uses the same setting as we do when $L_r = L_\psi = \sqrt{d} $ and $L_\phi = 1$. In that case we get the regret in Eqn.~(\ref{eqn:simple_regret}) which is $\sqrt{Hd}$-times worse in the leading term than \citet{jin2019provably}.

In terms of feature dimension $d$, this matches the $\sqrt{d}$ gap in linear bandits between the best bounds for a TS-based algorithm (with regret $\widetilde O(d^{3/2}\sqrt{T})$)  \citep{abeille2017linear} and the best bounds for an optimistic algorithm (with regret $\widetilde O(d\sqrt{T})$) \citep{Abbasi11}. This happens because the proof techniques for Thompson sampling require the perturbations to have sufficient variance to guarantee optimism (and thus exploration) with some probability. 
For a geometric interpretation of this, see \citet{abeille2017linear}. For $H$-horizon MDPs, the total system dimensionality is $dH$, and therefore the extra $\sqrt{dH}$ factor is expected.

\textbf{Comparison with \citet{russo2019worst}.} Recently, \citet{russo2019worst} has analyzed \textsc{Rlsvi} in tabular finite horizon MDPs. While the core algorithm is similar, function approximation does introduce challenges that required changing \textsc{Rlsvi} by, e.g., introducing the default function. While in \citet{russo2019worst} the value function can be bounded in high probability thanks to the non-expansiveness of the Bellman operator associated to the estimated model, in our case this has to be handled explicitly. We think that the use of a default optimistic value function could yield better horizon dependence for \textsc{Rlsvi} in tabular settings, though this would require changing the algorithm.

\section{Proof Outline}
In this section we outline the proof of our regret bound for \textsc{opt-rlsvi}. The four main ingredients are: 1) a one-step expansion of the action-value function difference $\overline Q_{tk}-Q^{\pi_k}_t$ in terms of the next-state value function difference; 2) a high probability bound on the noise and pseudonoise; 3) showing that the algorithm is optimistic with constant probability; 4) combining to get the regret bound.%
For the sake of clarity, we will assume no misspecification ($\epsilon = 0$), no regularization ($\lambda = 0$), and a nonsinigular design matrix $\Sigma_{tk} = \sum_{i=1}^{k-1} \phi_{ti}\phi_{ti}^\top$. The complete proof is reported in the appendix.

\subsection{One-Step Analysis of Q functions}

In this section we do a ``one-step'' analysis to decompose the difference in Q functions in the case where $\Smallcondition{tk}{\phi_t(s,a)}$ so that $ \overline Q_{tk}$ is linear in the features. The decomposition has three parts: environment noise, pseudonoise, and the difference in value functions at step $ t+1$. It reads $(\overline Q_{tk} - Q_t^\pi)(s,a) =$
\begin{align}\label{eqn:OneStepAnalysis}
	    \phi_t(s,a)^\top (\overline \eta_{tk} + \overline \xi_{tk}) + \E_{s'|s,a}(\overline V_{t+1,k} - V_{t+1}^\pi)(s')
\end{align}
where $\overline\eta_{tk}$ is the projected environment noise defined below in Eqn.~\eqref{eqn:etaDef}.
The complete version of the decomposition is Lem.~\ref{rlsvi:lemQ_decomp} in the appendix, while here we give an informal proof sketch of this fact.

First, since we are assuming that $\Smallcondition{tk}{\phi_t(s,a)}$ and $\epsilon = 0$, we can apply Def.~\ref{rlsvi:defQdef} and Prop.~\ref{rlsvi:proplinearQ_maintext} to write:
\begin{align}
    (\overline Q_{tk} - Q_t^\pi)(s,a) = \phi_t(s,a)^\top (\overline{\theta}_{tk} - \theta_t^\pi).
\end{align}
Decomposing $ \overline \theta_{tk} = \widehat \theta_{tk} + \overline{\xi}_{tk} $ immediately shows how the pseudonoise $ \overline{\xi}_{tk}$ appears in Eqn.~(\ref{eqn:OneStepAnalysis}). Now we need to handle the regression term:
\begin{align}
\label{eqn:thetahatMain}
    \widehat \theta_{tk} \defeq \Sigma_{tk}^{-1} \sum_{i=1}^{k-1}\phi_{ti}(r_{ti} + \overline V_{t+1,k}(s_{t+1,i})).
\end{align}
To handle this, we need to make an expectation over $ s'$ given $s_{ti}, a_{ti}$ (the experienced state and action in timestep $t$ of episode $i$) appear in each term of the sum so that the value function term will become linear in $ \phi_{ti}$. To do this, we define the one-step environment noise with respect to $ \overline V_{t+1, k} $ as
\begin{align}
    \overline\eta_{tk}(i) \defeq  \overline{V}_{t+1,k}(s_{t+1,i}) - \E_{s'|s_{ti},a_{ti}}[\overline{V}_{t+1,k}(s')],
\end{align}
Then we define the projected environment noise as:
\begin{align*}\numberthis{\label{eqn:etaDef}}
    \overline{\eta}_{tk} \defeq \Sigma_{tk}^{-1} \sum_{i=1}^{k-1} \phi_{ti} \overline\eta_{tk}(i).
\end{align*}
Putting this into the definition of $ \widehat \theta_{tk}$ from Eqn.~(\ref{eqn:thetahatMain}),
\begin{align*}
    \widehat \theta_{tk} &=\Sigma_{tk}^{-1} \sum_{i=1}^{k-1}\phi_{ti}(r_{ti} + \E_{s'|s_{ti}, a_{ti}}[\overline V_{t+1,k}(s')] + \overline \eta_{tk}(i)) \\
    &= \overline{\eta}_{tk} + \Sigma_{tk}^{-1} \sum_{i=1}^{k-1}\phi_{ti}(r_{ti} + \E_{s'|s_{ti}, a_{ti}}[\overline V_{t+1,k}(s')]).
\end{align*}
But now we note that this reward plus expected value function is linear in the features (thanks to Prop.~\ref{rlsvi:proplinearQ_maintext}), so we can rewrite the second term as
\begin{align}
    &\Sigma_{tk}^{-1} \sum_{i=1}^{k-1}\phi_{ti}\phi_{ti}^\top \(\theta^r + \int_{s'}\psi_t(s')\overline V_{t+1,k}(s')\)\\
    &\qquad = \theta^r + \int_{s'}\psi_t(s')\overline V_{t+1,k}(s').
\end{align}
Finally, comparing with the definition of $ \theta_t^\pi$ (Eqn.~(\ref{eqn:linearQ})) we see that the $ \theta^r$ terms cancel and we get
\begin{align*}
    \overline\theta_{tk} - \theta_t^\pi = \overline\xi_{tk} + \overline \eta_{tk} + \int_{s'}\psi_t(s') (\overline V_{t+1,k} - V_{t+1}^\pi)(s').
\end{align*}
Premultiplying by $\phi_t(s,a)^\top$ gives Eqn.~(\ref{eqn:OneStepAnalysis}).

\subsection{High Probability Bounds on the Noise}
\label{euler:sec:NoiseBoundMain}

To ensure that our estimates concentrate around the true $ Q $ functions, we need to ensure that the $ \overline \eta_{tk}$ and $ \overline{\xi}_{tk}$ are not too large.
This is achieved with similar ideas of self-normalizing processes as is done for linear bandits \citep{Abbasi11}, with an additional union bound over possible value functions $ \overline V_{t+1,k}$ which depend on $\overline \theta_{tk}$ and $\Sigma^{-1}_{tk}$.
In the end, we prove in Lem.~\ref{rlsvi:lemgood_prob} that indeed with high probability for any $ \phi$:
\begin{align}
\label{eqn:eta_bound}
    |\phi^\top \overline \eta_{tk}| \leq  \|\phi \|_{\Sigma^{-1}_{tk}} \| \overline \eta_{tk}\|_{\Sigma_{tk}} \leq \sqrt{\nu_k(\delta')} \|\phi\|_{\Sigma_{tk}^{-1}}
\end{align}
where $\sqrt{\nu_k(\delta')} = \widetilde O(dH)$ is defined fully in App. \ref{rlsvi:sec:good}.
While we defer the computation of the ``right'' amount of pseudonoise to the next subsection, here we mention that for the choice we make $ \overline{\xi}_{tk} \sim \mathcal{N}(0, H\nu_k(\delta')\Sigma_{tk}^{-1})$ we obtain w.h.p.:
\begin{align}
\label{eqn:xi_bound}
    |\phi^\top \overline \xi_{tk}|\leq \| \phi\|_{\Sigma^{-1}_{tk}}\| \overline \xi_{tk} \|_{\Sigma_{tk}}  \leq \sqrt{\gamma_k(\delta')} \|\phi\|_{\Sigma_{tk}^{-1}}
\end{align}
where $\sqrt{\gamma_k(\delta')} = \widetilde O((dH)^{3/2})$ is also defined fully in App. \ref{rlsvi:sec:good}.
Note the pseudonoise \emph{worst-case} bound is $\sqrt{Hd}$ worse than the corresponding environment noise.
\subsection{Stochastic Optimism and Random Walk}
\label{euler:sec:RandomWalkInMain}
We now want to show that \textsc{opt-rlsvi} injects enough pseudonoise that the estimated value function $\Vbar_{1k}(s_{1k})$ at the initial state $s_{1k}$ is optimistic with constant probability (see App.~\ref{rlsvi:sec:optimism}). We call this event $\Ok$:
\begin{align}
\Ok \defeq \Big\{(\overline V_{1k} - \Vstar_1)(s_{1k}) \geq 0 \Big\}.
\end{align}
Note that the optimal policy $\pistar$ maximizes $ Q^\star $ and not the $\Qbar$ computed by the algorithm and thus
\begin{align}
\label{eqn:pi_star_unroll}
    (\overline V_{1k} - \Vstar_1)(s_{1k}) \geq  (\Qbar_{1k} - Q_1^\star)(s_{1k},\pistar_1(s_{1k})).
\end{align}
Now, the goal is to leverage Eqn.~(\ref{eqn:OneStepAnalysis}) to inductively expand this inequality by unrolling a trajectory under the policy $ \pi^\star$. To access the result in Eqn.~(\ref{eqn:OneStepAnalysis}) we need to have $ \Smallcondition{1k}{\phi_1(s_{1k},\pistar_1(s_{1k}))}$. For now, we just assume that this is the case to motivate the idea. In that case, applying Eqn.~(\ref{eqn:OneStepAnalysis}) gives us
\begin{align*}
    &\( \Vbar_{1k} - \Vstar_1 \) (s_{1k}) \geq \phi_1(s_{1k},\pistar_1(s_{1k}))^\top \( \overline \xi_{1k} + \overline \eta_{1k} \) \\
 & \qquad + \E_{s'|s_{1k}, \pi^\star_1(s_{1k})}[ \(\Vbar_{2k} - \Vstar_{2} \)(s')].\numberthis{}
\end{align*}
Now we can inductively apply the same reasoning to the term inside of the expectation (assuming that we always get features with small $\Sigma^{-1}$-norm). Using $ x_t $ to denote the states sampled under $\pi^\star$ to avoid confusion with $ s_{tk} $ observed by the algorithm, we get
\begin{align}
\label{eqn:WalkInMain}
    \geq \sum_{t=1}^H \E_{x_t \sim \pi^\star|s_{1k}} \left[ \phi_t(x_t, \pi^\star_t(x_t))^\top (\overline\xi_{tk} + \overline \eta_{tk}) \right]
\end{align}
Since these trajectories over $ x $ come from $ \pi^\star$ and the environment, they do not depend on the algorithm's policy and with respect to the pseudonoise $ \overline \xi$, they are non-random. If we let $ \phi_t^\star$ denote $ \E_{x_t \sim \pi^\star|s_{1k}} \phi_t(x_t, \pi^\star(x_t))$, and apply Eqn.~(\ref{eqn:eta_bound}) we get with probability at least $ 1 - \delta$ that:
\begin{align*}
    &\sum_{t=1}^H (\phi_t^\star)^\top (\overline\xi_{tk} + \overline \eta_{tk}) \geq \sum_{t=1}^H [(\phi_t^\star)^\top \overline\xi_{tk} -  \sqrt{\nu_k(\delta')}\| \phi_t^\star \|_{\Sigma_{tk}^{-1}}]\\
    & \geq \sum_{t=1}^H (\phi_t^\star)^\top \overline\xi_{tk} - \sqrt{H\nu_k(\delta')} \(\sum_{t=1}^H  \| \phi_t^\star \|^2_{\Sigma_{tk}^{-1}}\)^{1/2}\numberthis{\label{eqn:optimsim_cs}}
\end{align*}
where the second inequality is Cauchy-Schwarz.

Note that the only randomness in this quantity comes from the pseudonoise we inject.
We can think of this sum as a one-dimensional normal random walk over  $ H $ steps with a negative bias.
Moreover, if we chose each $ \overline{\xi}_{tk} \sim \mathcal{N}(0, H\nu_k(\delta')\Sigma_{tk}^{-1})$, we know that
\begin{align}
    \sum_{t=1}^H (\phi_t^\star)^\top \overline\xi_{tk} \sim \mathcal N \(0, \sum_{t=1}^H H \nu_k(\delta')\|\phi_t^\star\|^2_{\Sigma_{tk}^{-1}}\).
\end{align}
Comparing this with Eqn.~(\ref{eqn:optimsim_cs}) we can immediately see that the standard deviation of the sum of pseudonoise terms is exactly the bound on the bias induced by the high probability bound on the sum of the environment noise $ \overline\eta_{tk}$. Thus we can conclude that
\begin{align}
    \Pro \( (\overline V_{1k} - V_1^\star)(s_{1k}) \geq 0\) \geq \Phi(-1)
\end{align}
where $ \Phi$ is the normal CDF. This is just the result that we are looking for. However, this presentation avoided the technicalities of handling the cases where $ \|\phi_t(x_t, \pi^\star_t(x_t))\|_{\Sigma_{tk}^{-1}} > \alpha_L$ and $ \overline Q_{tk}$ takes the default value. At a high level the default value is optimistic and so it cannot reduce the probability of optimism. This is handled carefully in Lem.~\ref{rlsvi:lemOptimisticRecursion} and \ref{lem:Optimism} of the appendix,
where we obtain a recursion structurally similar to Eqn.~(\ref{eqn:optimsim_cs}) albeit with a less interpretable definition of $\phi^\star_t$.
One important detail is that our choice of when to default does not depend on the $\overline{\xi}_{tk}$ and is thus non-random with respect to the pseudonoise.

\subsection{High Probability Regret Bound}
\label{euler:sec:RegretBoundMain}
In this section we provide a high level sketch of the main argument that allows us to obtain a high probability regret bound for \textsc{opt-rlsvi} under Asm.~\ref{rlsvi:assMainAssumption}. In particular, we assume that the ``good event'' holds, which lets us use the bounds in Eqn.~(\ref{eqn:eta_bound}) and (\ref{eqn:xi_bound}).

First, we recall the definition of regret up to episode $K$ from the preliminaries and further add and subtract the randomized value functions $\overline V_{1k}$ to get that $ \textsc{Regret(K)}$ decomposes as
\begin{align*}
 \numberthis{\label{eqn:RegretDef}} \sumk\Big( \underbrace{\Vstar_{1} - \overline V_{1k}}_{\text{Pessimism}} + \underbrace{\overline V_{1k} - V^{\pi_k}_{1}}_{\text{Estimation}} \Big) (s_{1k})
\end{align*}
\subsubsection{Bound on estimation}
\label{euler:sec:ConcentrationBoundMain}
We need to distinguish between cases where $\Smallcondition{tk}{\phi_{tk}}$, which we will denote by $\mathcal S_{tk}$ for small feature, or not, which we will denote by $\mathcal S^c_{tk}$ for its complement. Under $ \mathcal S_{tk}$ linearity of the representation can be used via Eqn.~(\ref{eqn:OneStepAnalysis}) and under $ \mathcal{S}^c_{tk}$ we can use the trivial upper bound of $H$ on the difference in values:
\begin{align*}
	& \( \overline V_{1k} - V_1^{\pi_k} \)(s_{1k})\leq H\1\{\mathcal S^c_{1k}\} + \numberthis{}\\
	& \Big( \phi_{1k}^\top\(\overline \xi_{1k} + \overline\eta_{1k} \) +\underbrace{\E_{s' \mid s_{1k},a_{1k}} [ \( \overline V_{2k} - V_2^{\pi_k}\)(s') ]}_{= \dot\zeta_{1k} + (\overline V_{2k} - V_2^{\pi_k})(s_{2k})}\Big)\1\{\mathcal S_{1k}\}
\end{align*}
where $ \dot \zeta_{tk} \defeq \1\{\mathcal S_{1k}\} \big( \E_{s' \mid s_{tk},a_{tk}}  \( \overline V_{t+1,k} - V_{t+1}^{\pi_k}\)(s') -(\overline V_{t+1,k} - V_{t+1}^{\pi_k})(s_{t+1,k}) \big)$ is a bounded martingale difference sequence on the good event.
Induction and summing over $k$ eventually yields:
\begin{align*}
	& \leq \sumk \sum_{t=1}^{H} \underbrace{H\1\{\mathcal S^c_{tk}\}}_{\text{Warmup}}  + \underbrace{\phi_{tk}^\top\(\overline \xi_{tk} + \overline \eta_{tk} \) \1\{\mathcal S_{tk}\}}_{\text{Linear Regime}} + \underbrace{\dot \zeta_{tk} \1\{\mathcal S_{tk}\}}_{\text{Martingale}}.
\end{align*}
The martingale term can be bounded with high probability by $ \tilde O(H\sqrt{T})$ using Azuma-Hoeffding.

The first term measures regret during ``warmup'', when the algorithm cannot guarantee that the value function estimates are bounded and needs to use the default function. In Lem.~\ref{rlsvi:lemWarmup} we bound it and obtain:
\begin{align}
\widetilde O\(\frac{H^2d}{\alpha_L^2}\) = \widetilde O \(H^5d^4\)
\end{align}
which is $\sqrt{T}$-free and is thus a lower order term.

For the dominant linear regime term we can use the high probability bounds from Eqn.~\eqref{eqn:eta_bound} and~\eqref{eqn:xi_bound} along with two applications of Cauchy-Schwarz:
\begin{align*}
&  \leq \sumk \sum_{t=1}^{H} \| \phi_{tk}\|_{\Sigma^{-1}_{tk}}\( \sqrt{\gamma_k(\delta')} +  \sqrt{\nu_k(\delta')} \) \numberthis{} \\
& \leq \sqrt{K} \times \sum_{t=1}^H \underbrace{\sqrt{\sumk \| \phi_{tk}\|^2_{\Sigma^{-1}_{tk}}}}_{\widetilde O(\sqrt{d})} \times  \Big(\underbrace{\vphantom{\sqrt{\sumk}}\sqrt{\gamma_K(\delta')}}_{\widetilde O(H^{3/2}d^{3/2})} + \underbrace{\vphantom{\sqrt{\sumk}} \sqrt{\nu_K(\delta')}}_{\widetilde O(Hd)} \Big)
\end{align*}
This final bound on the sum of the squared norm of the features is a standard quantity that arises in linear bandit computations \citep{Abbasi11}.
We can see that the estimation term gives the same regret bound reported in the Thm.~\ref{main:MainResult}. Now we show that the pessimism term is of the same order.

\subsubsection{Bound on Pessimism}
\label{euler:sec:PessimismBoundMain}
For optimistic algorithms the pessimism term of the regret $\sumk (\Vstar_{1} - \overline V_{1k}^{\pi_k})(s_{1k})$ is negative by construction; here we need to work a little more. As seen above, the algorithm has at least a \emph{constant} probability of being optimistic. When it is, it makes progress similar to a deterministic optimistic  algorithm, and when it is not, it is still choosing a reasonable policy (using shrinking confidence intervals) so that the mistakes it makes become less and less severe. Ultimately, we would like to transform the pessimism term into an estimation argument that we can handle as before. So, we first upper bound $\Vstar_{1}$ and then lower bound $\overline V_{1k}$ by randomized value functions with specific choices for the pseudonoise. As more samples are collected, the pseudonoise shrinks and the estimates converge.

\paragraph{Upper Bound on $\Vstar_1$.}
Consider drawing $\widetilde \xi_{tk}$'s defined as independent and identically distributed copies of the $\overline \xi_{tk}$'s. Let $\widetilde{\mathcal O}_k$ be the event that in episode $k$ the algorithm obtains an optimistic value function $\widetilde V_{1k}$ using these $\widetilde \xi_{tk}$ in place of $ \overline{\xi}_{tk}$. Explicitly,
\begin{align}
	\widetilde{\mathcal O}_k = \{(\widetilde V_{1k}-\Vstar_{1k})(s_{1k}) \geq 0 \}.
\end{align}
Note that since the $\widetilde\xi_{tk} $ are iid copies of the $ \overline{\xi}_{tk}$ we have that $ \Pro(\widetilde{\mathcal O}_k) $ is equal to $\Pro(\mathcal O_k) = \Phi(-1)$ from Sec.~\ref{euler:sec:RandomWalkInMain}. Taking conditional expectation $\E_{\widetilde \xi | \widetilde{\mathcal O}_k}$ over the $\widetilde \xi_{tk}$ for $ t \in [H]$ gives us an upper bound:
 \begin{align}
 \label{eqn:UBmain}
\Vstar_{1k}(s_{1k}) & \leq \E_{\widetilde\xi|\widetilde{\mathcal O}_k} \widetilde V_{1k}(s_{1k})
\end{align}
by definition of the event $\widetilde{\mathcal O}_k$.
\paragraph{Lower Bound on $\overline V_{1k}$.}
Under the high probability bound on the pseudonoise of Eqn.~(\ref{eqn:xi_bound}) we consider the below optimization program over the \emph{optimization variables} $\xi_{tk}$'s, which are constrained to satisfy the same bound on the pseudonoise of Eqn.~(\ref{eqn:xi_bound}):
\begin{align*}
\numberthis{\label{eqn:OptimProgramMain}}
& \min_{\{\xi_{tk} \}_{t=1,\dots,H}}  V^{\xi}_{1k}(s_{1k}) \\
& \|\xi_{tk}\|_{\Sigma_{tk}}  \leq \sqrt{\gamma_{k}(\delta')}, \quad \forall t \in [H]
\end{align*}
where $ V^\xi_{1k}$ is analogous to $ \overline{V}_{1k}$ derived from our algorithm, but with the optimization variables $ \xi_{tk}$ in place of $ \overline{\xi}_{tk}$.
Solving the program above would give a value function $\underline{V}_{1k}$ such that:
\begin{align}
\label{eqn:LBmain}
\underline V_{1k}(s_{1k}) \leq \overline V_{1k}(s_{1k})
\end{align}
whenever the $ \overline{\xi}_{tk}$'s obey the high probability bound.

\paragraph{Putting it together.}
Now we chain the upper bound of Eqn.~(\ref{eqn:UBmain}) with the lower bound of Eqn.~(\ref{eqn:LBmain}):
\begin{align}
\label{eqn:main01}
\(\Vstar_{1k}-\overline V_{1k}\)(s_{1k})
& \leq \E_{\widetilde\xi| \widetilde{\mathcal O}_{k}}[( \widetilde V_{1k}-\underline V_{1k})(s_{1k})].
\end{align}
We can connect this conditional expectation with the probability of optimism to get to a concentration bound by applying the law of total expectation:
\begin{align*}
   \E_{\tilde \xi} [(\widetilde V_{1k} - &\underline V_{1k} )(s_{1k})] = \E_{\widetilde\xi|\widetilde{\mathcal O}_{k}}[ (\widetilde V_{1k} - \underline V_{1k} )(s_{1k})] \Pro(\widetilde{\mathcal O}_k)  \\
 & + \underbrace{\E_{\widetilde\xi|\widetilde{\mathcal O}_{k}^c} [(\widetilde V_{1k} - \underline V_{1k})(s_{1k})] \Pro(\widetilde{\mathcal O}_k^c)}_{ \geq 0}.\numberthis{}
\end{align*}
This inequality holds by the same reasoning as Eqn.~(\ref{eqn:LBmain}) with high probability since the $ \widetilde\xi_{tk}$ are also in the set over which $ \underline{V}_{1k}$ is minimized.
Dividing by $ \Pro(\widetilde{\mathcal O}_k)$ and chaining with Eqn.~(\ref{eqn:main01}) gives us:
\begin{align*}
\(\Vstar_{1k} - \overline V_{1k}\)(s_{1k}) \leq \E_{\widetilde \xi}[(\widetilde V_{1k} - \underline V_{1k} )(s_{1k})]/\Pro(\widetilde{\mathcal O}_k).
\end{align*}
Now, since the $\widetilde \xi_{tk}$ are iid copies of the $\overline \xi_{tk}$ that the algorithm computes we have that $ \E_{\widetilde\xi}[\widetilde V_{1k}(s_{1k})] = \E_{\overline\xi}[\overline V_{1k}(s_{1k})]$ and $ \Pro(\mathcal O_k) = \Pro(\widetilde{\mathcal O}_k)$. So we can define a martingale difference sequence $ \ddot \zeta_k \defeq \E_{\widetilde\xi}[\widetilde V_{1k}(s_{1k})] - \overline V_{1k}(s_{1k})$ and get our final bound on the pessimism as:
\begin{align}
\label{eqn:alongtheway}
\(\Vstar_{1k} - \overline V_{1k}\)(s_{1k}) \leq \frac{(\overline V_{1k}  - \underline V_{1k})(s_{1k})+ \ddot \zeta_k }{\Pro(\mathcal O_k)}.
\end{align}
When summing over the episodes $k\in[K]$, the martingale can be bounded with high probability by Azuma-Hoeffding as $\sumk \ddot \zeta_k = \widetilde O(H\sqrt{K})$. To bound the remaining term we add and subtract $ V_1^{\pi_k}$ to get:
\begin{align*}
   \(\sumk [(\overline V_{1k} - V_1^{\pi_k})(s_{1k}) + (V_1^{\pi_k} - \underline V_{1k})(s_{1k})]\) / \Pro(\Ok).
\end{align*}
Each of these is  bounded by arguments similar to those in Sec.~\ref{euler:sec:ConcentrationBoundMain}. We discuss this in detail in Lem.~\ref{rlsvi:lemPessimism}.

It is instructive to re-examine Eqn.~(\ref{eqn:alongtheway}), ignoring the martingale term. While the left hand side is negative for optimistic algorithms, for \textsc{opt-rlsvi} it is upper bounded by a difference in estimated value functions (which shrinks with more data) times the inverse probability of being optimistic $1/\Pro({\mathcal O}_k)$. In other words, roughly once every $1/\Pro(\mathcal O_k)$ episodes the algorithm is optimistic and exploration progress is made.

\section{Concluding Remarks}
This work proposes the first high probability regret bounds for (a modified version of) \textsc{Rlsvi} with function approximation, confirming its sound exploration principles. Perhaps unsurprisingly, we inherit an extra $\sqrt{dH}$ regret factor compared to an optimistic approach which can be explained by analogy to the bandit literature. Whether Thompson sampling-based algorithms need to suffer this extra factor compared to their optimistic counterparts remains a fundamental research question in exploration. Our work enriches the literature on provably efficient exploration algorithms with function approximation with a new algorithmic design as well as a new set of analytical techniques.

\subsection*{Acknowledgments}
Andrea Zanette is partially supported by the Total Innovation Fellowship program. David Brandfonbrener is supported by the Department of Defense (DoD) through the National Defense Science \& Engineering Graduate Fellowship (NDSEG) Program. We would also like to thank Haque Ishfaq for pointing out a small error in a previous version of the paper.
In addition, we thank Taehyun Hwang, Min-hwan Oh and Qiwen Cui for pointing out another small error in the proof.

\bibliography{rl}

\newpage
\appendix
\onecolumn
\addcontentsline{toc}{section}{Appendix} 
\part{Appendix} 
\parttoc
\newpage
\section{Notation}\label{sec:notation}
We provide this table for easy reference. Notation will also be defined as it is introduced.

We denote with $H$ the episode length, with $K$ the total number of episodes, and with $T = HK$ the time elapsed.
We denote with $k \in [K]$ the current episode, with $t \in [H]$ the current timestep. We use the subscript $tk$ to indicate the quantity at timestep $t$ of episode $k$ and $ t+1,k$ for the subsequent step.

\renewcommand{\arraystretch}{1.5}
\begin{longtable}{l l l}
\caption{Symbols}\\
\hline
$s_{tk}  $ & $\defeq$& state encountered in timestep $t$ of episode $k$\\
$a_{tk}  $ & $\defeq$& action taken by the algorithm in timestep $ t $ of episode $ k$ \\
$\phi_{tk} $ & $\defeq$& $\phi_t(s_{tk},a_{tk})$ \\
$r_{tk} $ & $\defeq$& $r_t(s_{tk},a_{tk})$ \\
$\lambda $ & $\defeq$& regularization parameter\\
$\Sigma_{tk} $ & $\defeq$ & $\sum_{i=1}^{k-1} \phi_{ti}\phi_{ti}^\top + \lambda I$\\
$\thetahat_{tk}$ & $\defeq$& $ \Sigma_{tk}^{-1}\(\sum_{i=1}^{k-1} \phi_{ti} [r_{ti} + \overline V_{t+1,k}(s_{t+1,i})]\)$ \\
$\delta'$ & $\defeq$ & $\delta / (16HK)$\\
$\sqrt{\beta_k(\delta')}$ & $\defeq$ & $c_1 Hd \sqrt{\log \(\frac{Hdk \max(1,L_\phi) \max(1,L_\psi) \max(1,L_r) \lambda}{\delta'}\)}$\\
$\sqrt{\nu_k(\delta')}$ & $ \defeq$ & $ \sqrt{\beta_k(\delta')}  + \sqrt{\lambda} L_\phi ( 3HL_\psi + L_r ) + 4 \epsilon H \sqrt{dk}$\\
$\sqrt{\gamma_{k}(\delta')}  $ & $\defeq$& $c_2\sqrt{dH \nu_{k}(\delta') \log(d/\delta')}$\\
$\overline \xi_{tk} $ & $\defeq$& Pseudonoise distributed as $\mathcal N(0,H\nu_{k}(\delta')\Sigma^{-1}_{tk}) $\\
$\overline \theta_{tk}$ & $\defeq$& $\thetahat_{tk} + \overline \xi_{tk}$ \\
$\alpha_U$ & $\defeq$ & $\frac{1}{4(\sqrt{\gamma_k(\delta')})} $\\
$\alpha_L$ & $\defeq$ & $\frac{\alpha_U}{2} $\\
$\qo{tk}(s,a)$ & $\defeq$& $
\begin{cases}
      \phi^\top \thetabar_{tk}, &\text{if } \Smallcondition{tk}{\phi_t(s,a)} \\
      H-t+1,  &\text{if } \Largecondition{tk}{\phi_t(s,a)}\\
      \frac{\alpha_U - \|\phi_t(s,a)\|_{\Sigma_{tk}^{-1}}}{\alpha_U - \alpha_L}\(\phi^\top \thetabar_{tk}\) + \frac{ \|\phi_t(s,a)\|_{\Sigma_{tk}^{-1}} - \alpha_L}{\alpha_U - \alpha_L}(H-t+1),  & \text{otherwise}
    \end{cases}
$\\
$\vo_{tk}(s)$ & $\defeq$& $\max_{a} \qo{tk}(s,a)$ \\
$\pi_{k}(s)$ & $\defeq$& policy executed by the algorithm in episode $ k$, i.e. $\argmax_{a} \qo{tk}(s,a)$ \\
$\mathcal S_{tk} $ & $\defeq$& Event $ \Big\{\Smallcondition{tk}{\phi_{tk}} \Big\}$ \\
$\mathcal S^c_{tk} $ & $\defeq$& Event $ \Big\{\NonSmallcondition{tk}{\phi_{tk}} \Big\}$ (complement of $\mathcal S_{tk}$)\\
$ L_\phi $ & $ \defeq$ & upper bound on $ \|\phi\|$ \\
$ L_\psi $ & $ \defeq$ & upper bound on $ \int_{s} \|\psi_t(s')\|$ for all $ t \in [H]$\\
$ L_r $ & $ \defeq$ & upper bound on $ \|\theta_r\|$\\
$ L_\theta $ & $ \defeq$ & upper bound on $\|\theta_t^\pi\|$ (equal to $ L_r + (H-1)L_\psi$) \\
$ \Delta_t^P(\cdot|s,a) $ & $ \defeq$ & $ \mathbb{P}_t(\cdot|s,a) - \phi(s,a)_t^\top \psi_t(\cdot)$\\
$ \Delta_t^r(s,a) $ & $ \defeq$ & $ r_t(s,a) - \phi(s,a)_t^\top \theta_t^{r}$\\
$ \epsilon $ & $ \defeq$ & bound on $|\Delta_t^r(s,a)|$ and $ \|\Delta_t^P(\cdot|s,a)\|_1$\\
$ \overline \eta_{tk}$ & $ \defeq$ & $ \Sigma_{tk}^{-1} \sum_{i=1}^{k-1} \phi_{ti}\bigg( \overline V_{t+1,k}(s_{t+1,i}) - \E_{s'|s_{it}, a_{it}}[\overline V_{t+1,k}(s')]\bigg)$\\
$ \overline \lambda_{tk}^\pi$ & $ \defeq$ & $-\lambda \Sigma_{tk}^{-1}\bigg(\int_{s'} \psi_t(s') (\overline V_{t+1,k} - V_{t+1}^\pi)(s') + \theta_t^\pi \bigg)$\\
$ \Delta_t^\pi(s,a)$ & $ \defeq$ & $ Q_t^\pi(s,a) - \phi_t(s,a)^\top \theta_t^\pi$\\
$ \overline m_{tk}^\pi$ & $ \defeq$ & $ \phi_t(s,a)^\top\Sigma_{tk}^{-1}\sum_{i=1}^{k-1} \phi_{ti} \bigg[\Delta_t^r(s_{ti}, a_{ti}) + \int_{s'}\Delta_t^P(s'|s_{ti}, a_{ti})\overline V_{t+1,k}(s') \bigg] + \Delta_{t}^\pi(s,a)$\\ & &$- \int_{s'} \Delta_t^P(s'|s,a)(\overline V_{t+1,k}- V_{t+1}^\pi)(s')$\\
$ \mathcal{H}_{tk}$ & $ \defeq$ & $\{s_{ij}, a_{ij}, r_{ij}: j \leq k,\quad i\leq t \text{ if } j=k \text{ else } i \leq H\}$\\
$ \overline{\mathcal{H}}_{tk}$ & $ \defeq$ & $\mathcal{H}_{Hk} \ \bigcup\  \{\overline{\xi}_{ik}: i \geq t\}$\\
$ \mathcal{G}^{\overline{\xi}}_{tk}$ & $ \defeq$ & $\bigg\{ |\phi_t(s,a)^\top \overline{\xi}_{tk}| \leq \sqrt{\gamma_k(\delta')} \|\phi_t(s,a)\|_{\Sigma_{tk}^{-1}}\bigg\}$ \\
$ \mathcal{G}^{\overline{\eta}}_{tk}$ & $ \defeq$ & $\bigg\{ |\phi_t(s,a)^\top \overline \eta_{tk}| \leq \sqrt{\beta_k(\delta')} \|\phi_t(s,a)\|_{\Sigma_{tk}^{-1}}\bigg\}$\\
$ \mathcal{G}^{\overline{\lambda}}_{tk}$ & $ \defeq$ & $\bigg\{\forall \ \pi, \quad |\phi_t(s,a)^\top\overline \lambda^\pi_{tk}| \leq \sqrt{\lambda} L_\phi ( 3HL_\psi + L_r ) \|\phi_t(s,a)\|_{\Sigma_{tk}^{-1}}\bigg\}$\\
$ \mathcal{G}^{\overline{m}}_{tk}$ & $ \defeq$ & $\bigg\{\forall \ \pi, \quad |\overline m_{tk}^\pi (s,a)| \leq 4 \epsilon H (\sqrt{dk}\|\phi_t(s,a)\|_{ \Sigma_{tk}^{-1}} + 1)  \bigg\}$ \\
$ \mathcal{G}^{\overline{Q}}_{tk}$ & $ \defeq$ & $\bigg\{\forall
        \ s,a,\quad |(\overline Q_{tk} - Q_t^\star)(s,a)| \leq H - t + 1\bigg\}$ \\
$ \overline {\mathcal{G}}_{tk}$ & $ \defeq$ & $\{  \mathcal{G}^{\overline{\xi}}_{tk} \cap \mathcal{G}^{\overline \eta}_{tk} \cap \mathcal{G}^{\overline \lambda}_{tk} \cap \mathcal{G}^{\overline m}_{tk} \cap \mathcal{G}^{\overline Q}_{tk}\}$\\
$ \overline {\mathcal{G}}_{k}$ & $ \defeq$ & $\bigcap_{ t \in [H]} \overline{\mathcal{G}}_{tk}$\\
$\widetilde \xi_{tk} $ & $\defeq$& i.i.d. copy of the pseudonoise $\widetilde \xi_{tk}$, useful for the regret proof. \\
& & All overline quantities can be translated to tilde \\
& & by exchanging pseudonoise variables in the value iteration.\\
$ \widetilde {\mathcal{O}}_{k}$ & $ \defeq$ & $\Big\{\(\widetilde V_{1k} - \Vstar_{1}\)(s_{1k}) \geq - 4 H^2\epsilon\Big\}$\\
\label{rlsvi:tab:MainNotation}
\end{longtable}

\newpage

\section{Assumptions}\label{sec:assumptions}

In this section we formally present the main assumption that the MDP is approximately low-rank and show that the definition immediately implies the existence of approximately linear $ Q $ functions for any policy. Moreover, the corresponding parameters to these $ Q$ functions have bounded norm.

\begin{assumption}[$\epsilon$-approximate low-rank MDP]\label{rlsvi:assapprox_lowrank}\citep{jin2019provably, yang2019reinforcement}
For any $ \epsilon \leq 1$, an MDP $ (\mathcal{S}, \mathcal{A}, H, \mathbb{P}, r)$ is $ \epsilon$-approximate low-rank with feature maps $ \phi_t: \mathcal{S} \times \mathcal{A} \to \R^d $ if for every $ t \in [H]$ there exists an unknown function $ \psi_t : \mathcal{S} \to \R^d$ and an unknown vector $ \theta_t^{r} \in \R^d$ such that
\begin{align}
    \|\mathbb{P}_t(\cdot|s,a) - \phi(s,a)_t^\top \psi_t(\cdot) \|_{1} \leq \epsilon, \qquad |r_t(s,a) - \phi(s,a)_t^\top \theta_t^{r}| \leq \epsilon.
\end{align}
Moreover assume the bounds
\begin{enumerate}
    \item $\|\phi_t(s,a) \| \leq L_\phi$ for all $ (s,a) \in \mathcal{S} \times \mathcal{A}$ and $ t \in [H]$. 
    \item $ \int_{\mathcal{S}}\|\psi_t(s)\| \leq L_\psi$ for all $ t \in [H]$. 
    \item $ \| \theta_t^{r}\| \leq L_{r}$ for all $ t \in [H]$. 
\end{enumerate}
\end{assumption}

\begin{definition}[Misspecification]\label{rlsvi:defmisspec}
We can define the following misspecification quantities
\begin{align}
    \Delta_t^P(\cdot|s,a) \defeq \mathbb{P}_t(\cdot|s,a) - \phi(s,a)_t^\top \psi_t(\cdot)&,\qquad \|\Delta_t^P( \cdot| s,a) \|_1 = \int_{s'} \big|\Delta_t^P(s' | s,a)\big| \leq \epsilon\\
    \Delta_t^r(s,a) \defeq r_t(s,a) - \phi(s,a)_t^\top \theta_t^{r}&,\qquad |\Delta_t^r(s,a)| \leq \epsilon
\end{align}
where the inequalities follow from the Assumption \ref{rlsvi:assapprox_lowrank}.
\end{definition}

\begin{corollary}[Linear Q functions]\label{cor:linear_Q}
For any policy $ \pi$, there exist some $ \theta_t^\pi \in \R^d$ for all $ t \in [H]$ such that for all $ s,a$
\begin{align}
    |Q^{\pi}_t(s,a) - \phi(s,a)_t^\top \theta_t^\pi| \leq (H - t +1)\epsilon.
\end{align}
Moreover, $ \|\theta_t^\pi\| \leq L_r + (H-t) L_\psi \defeq L_\theta$. 
\end{corollary}
\begin{proof}
Since $Q_t^\pi(s,a) = \phi(s,a)^\top \left( \theta_t^r + \int \psi(s') V^\pi_{t+1} (s') \mathrm{d}s' \right)$, we set
\begin{align} \label{eqn:thetapi}
    \theta_t^\pi = \theta_t^r + \int_{s'}\psi_t(s')V_{t+1}^\pi(s')
\end{align}
Note that by the assumption that the rewards are in $ [0,1] $ the true value functions $ V_t^\pi$ are always in $ [0,H- t + 1]$. By the triangle inequality and Bellman equation followed by an application of Definition \ref{rlsvi:defmisspec}
\begin{align}
    |Q^{\pi}_t(s,a) - \phi_t(s,a)^\top \theta_t^\pi| &\leq |r_t(s,a) - \phi(s,a)_t^\top \theta_t^r| + \bigg|\E_{s'|s,a}[V^{\pi}_{t+1}(s')] - \phi_t(s,a)^\top \int_{s'}\psi_t(s')V_{t+1}^\pi(s')\bigg|\\
    &\leq \epsilon + \bigg| \int_{s'} (P_t(s'|s,a) -  \phi_t(s,a)^\top\psi_t(s'))V_{t+1}^\pi(s')\bigg|\\
    &\leq \epsilon + \|V_{t+1}^\pi \|_{\infty} \|\Delta_t^P(\cdot|s,a)\|_{1} \leq \epsilon + (H -t)\epsilon = (H-t+1)\epsilon
\end{align}
To prove the second part of the statement, note that by the triangle inequality and Assumption \ref{rlsvi:assapprox_lowrank}
\begin{align}
    \| \theta_t^\pi\| \leq \| \theta_t^r \| + \|\int_{s'}\psi_t(s')V_{t+1}^\pi(s')\| \leq L_r + \|V_{t+1}^\pi\|_\infty L_\psi \leq L_r + (H-t)L_\psi.
\end{align}
\end{proof}

\begin{definition}[Optimal parameters]
    We can denote the parameters associated with the optimal policy $ \pi^\star$ as $ \theta_t^\star = \theta_t^{\star, P} + \theta_t^{r}$.
\end{definition}

\newpage
\section{Decomposition of Unclipped Q-values}\label{sec:decomp}
In this section we prove the main decomposition lemma that will be useful throughout. The lemma decomposes the difference between the function defined by the estimated $ \overline{\theta}_{tk}$ and the true $ Q^\pi$ for any policy $ \pi$ into several parts: the {\color{NavyBlue}expected difference of corresponding value functions at the next state}, the {\color{BrickRed}projected environment noise}, the {\color{ForestGreen}pseudonoise}, a  {\color{DarkOrchid}term due to the regularizer $ \lambda$} and a term due to the {\color{RawSienna}misspecification} (i.e. the $ \epsilon $ error) of the low-rank MDP.

These terms are defined in the following notation:
\begin{align}
    \label{eq:def_etatk}
    \overline \eta_{tk} &\defeq \Sigma_{tk}^{-1} \sum_{i=1}^{k-1} \phi_{ti}\bigg( \overline V_{t+1,k}(s_{t+1,i}) - \E_{s'|s_{ti}, a_{it}}[\overline V_{t+1,k}(s')]\bigg)\\
    \label{eq:def_lambdatkpi}
    \overline \lambda_{tk}^\pi &\defeq -\lambda \Sigma_{tk}^{-1}\bigg(\int_{s'} \psi_t(s') (\overline V_{t+1,k} - V_{t+1}^\pi)(s') + \theta_t^\pi \bigg)\\
    \Delta_t^\pi(s,a) &\defeq Q_t^\pi(s,a) - \phi_t(s,a)^\top \theta_t^\pi \\
    \label{eq:def_mtkpi}
    \overline m_{tk}^\pi(s,a)
    &\defeq   \phi_t(s,a)^\top\Sigma_{tk}^{-1}\sum_{i=1}^{k-1} \phi_{ti} \bigg[\Delta_t^r(s_{ti}, a_{ti}) + \int_{s'}\Delta_t^P(s'|s_{ti}, a_{ti})\overline V_{t+1,k}(s') \bigg] + \Delta_{t}^\pi(s,a)\\&\qquad- \int_{s'} \Delta_t^P(s'|s,a)(\overline V_{t+1,k}- V_{t+1}^\pi)(s')
\end{align}

\begin{lemma}[Decomposition of unclipped Q-values]\label{rlsvi:lemQ_decomp} For $t \in [H]$ and any policy $ \pi$:
\begin{align}
    \phi_t(s,a)^\top \overline \theta_{tk} - Q_t^\pi(s,a) &=  {\color{NavyBlue}\E_{s'|s,a}[\( \overline V_{t+1,\pk} - V^\pi_{t+1} \)(s')]}  + \phi_t(s,a)^\top (
    {\color{BrickRed}\overline \eta_{tk}} +
    {\color{ForestGreen}\overline\xi_{tk}} +
    {\color{DarkOrchid}\overline \lambda_{tk}^\pi}) +
    {\color{RawSienna}\overline m_{tk}^\pi(s,a)}
\end{align}
where $\E_{s'|s,a}[\cdot] = \E_{s'\sim \mathbb{P}_t(\cdot|s,a)}[\cdot]$ and the index $ t$ will be clear from context.
\end{lemma}

\begin{proof}

By Corollary \ref{cor:linear_Q} we have:
\begin{align}
\phi_t(s,a)^\top \overline\theta_{tk}  - Q_t^\pi(s,a) &= \phi_t(s,a)^\top (\overline\theta_{tk} - \theta_{t}^\pi) + \Delta_{t}^\pi(s,a)
\end{align}

By substituting the definition of $\overline \theta_{tk}$ and the linear regression, we get:
\begin{align}
&=  \phi_t(s,a)^\top \Bigg(\overline\xi_{tk}+
\underbrace{
\Sigma_{tk}^{-1} \sum_{i=1}^{k-1} \phi_{ti} (r_{ti} + \overline V_{t+1,k}(s_{t+1,i}))
}_{= \widehat{\theta}_{tk}}
-  \theta^\pi_{t} \Bigg) + \Delta_{t}^\pi(s,a)
\end{align}
Moving $\theta_t^\pi$ inside the sum by multiplying by $ \Sigma_{tk}^{-1} \Sigma_{tk} = I$ we get
\begin{align}
    &=  \phi(s,a)^\top \(\overline\xi_{tk} + \Sigma_{tk}^{-1} \( - \lambda \theta_t^\pi + \sum_{i=1}^{k-1} \phi_{ti} \Big(r_{ti} + \overline V_{t+1,k}(s_{t+1,i}) - \phi_{ti}^\top \theta_t^\pi\Big) \)\) + \Delta_{t}^\pi(s,a).
\end{align}
Now we expand $ \phi_{ti}^\top \theta_t^\pi = \phi_{ti}^\top (\theta_t^r + \int_{s'}\psi(s')V_{t+1}^\pi(s'))$ (see Eq.~\ref{eqn:thetapi})
\begin{align}
    &=  \phi_t(s,a)^\top \(\overline\xi_{tk} + \Sigma_{tk}^{-1} \( - \lambda \theta_t^\pi + \sum_{i=1}^{k-1} \phi_{ti} \Bigg[r_{ti} + \overline V_{t+1,k}(s_{t+1,i}) - \phi_{ti}^\top \big(\theta_t^r + \int_{s'}\psi_t(s')V_{t+1}^\pi(s')\Big)\Bigg] \)\) \\
    & + \Delta_{t}^\pi(s,a).
\end{align}
Next we add and subtract $ \E_{s'|s_{ti},a_{ti}}[\overline V_{t+1,k}(s')] - \phi_{ti}^\top \int_{s'}\psi(s')\overline V_{t+1,k}(s')$ and rearrange terms to get
\begin{align}
    &=  \phi_t(s,a)^\top \bigg(\overline\xi_{tk}   - \lambda \Sigma_{tk}^{-1} \theta_t^\pi \\
    & + \Sigma_{tk}^{-1}\sum_{i=1}^{k-1} \phi_{ti} \underbrace{\bigg[r_{ti} - \phi_{ti}^\top \theta_t^r + \E_{s'|s_{ti},a_{ti}}\Big[\overline V_{t+1,k}(s')\Big] - \phi_{ti}^\top \int_{s'}\psi_t(s')\overline V_{t+1,k}(s')\bigg]}_{=\Delta_t^r(s_{ti}, a_{ti}) + \int_{s'}\Delta_t^P(s'|s_{ti}, a_{ti})\overline V_{t+1,k}(s') }\\&\qquad\qquad\qquad\qquad\qquad + \underbrace{\Sigma_{tk}^{-1} \sum_{i=1}^{k-1} \phi_{ti}\bigg[\overline V_{t+1,k}(s_{t+1,i}) - \E_{s'|s_{ti},a_{ti}}\Big[\overline V_{t+1,k}(s')\Big]\bigg]}_{\overline \eta_{tk}}\\&\qquad\qquad\qquad\qquad\qquad + \Sigma_{tk}^{-1}\sum_{i=1}^{k-1} \phi_{ti}\bigg[\phi_{ti}^\top \int_{s'}\psi_t(s')(\overline V_{t+1,k} - V_{t+1}^\pi)(s')\bigg] \bigg) + \Delta_{t}^\pi(s,a).
\end{align}
We can add and subtract a regularizer term and cancel $
\Sigma_{tk}^{-1}\Sigma_{tk}$ to get
\begin{align}
    &=  \phi_t(s,a)^\top \bigg(\overline\xi_{tk} + \overline{\eta}_{tk}  + \int_{s'}\psi_t(s')(\overline V_{t+1,k} - V_{t+1}^\pi)(s') \\&\qquad\qquad\qquad\qquad\qquad   - \underbrace{\lambda \Sigma_{tk}^{-1} \bigg[\theta_t^\pi + \int_{s'}\psi_t(s')(\overline V_{t+1,k} - V_{t+1}^\pi)(s')\bigg]}_{\overline \lambda_{tk}^\pi}\\&\qquad\qquad\qquad\qquad\qquad+ \Sigma_{tk}^{-1}\sum_{i=1}^{k-1} \phi_{ti} \bigg[\Delta_t^r(s_{ti}, a_{ti}) + \int_{s'}\Delta_t^P(s'|s_{ti}, a_{ti})\overline V_{t+1,k}(s') \bigg] \bigg) + \Delta_{t}^\pi(s,a)
\end{align}
Finally we replace the integral by the true expectation plus a misspecification term
\begin{align}
    &=  \phi_t(s,a)^\top (\overline\xi_{tk} + \overline{\eta}_{tk} + \overline{\lambda}_{tk}^\pi) + \E_{s'|s,a}[(\overline V_{t+1,k} - V_{t+1}^\pi)(s')] + \overline m_{tk}^\pi(s,a)
    \end{align}
    where
    \begin{align} \overline m_{tk}^\pi(s,a) & =  - \int_{s'} \Delta_t^P(s'|s,a)(\overline V_{t+1,k} - V_{t+1}^\pi)(s') \\
    & + \phi_t(s,a)^\top\Sigma_{tk}^{-1}\sum_{i=1}^{k-1} \phi_{ti} \bigg[\Delta_t^r(s_{ti}, a_{ti}) + \int_{s'}\Delta_t^P(s'|s_{ti}, a_{ti})\overline V_{t+1,k}(s') \bigg] + \Delta_{t}^\pi(s,a)
\end{align}
\end{proof}

\newpage
\section{Defining the Good Event}\label{rlsvi:sec:good}
In this section we formally define the filtrations that compose the history of the algorithm at any point during its runtime. Then we define the values $ \beta_k(\delta'), \nu_k(\delta'),$ and $\gamma_k(\delta')$ that are used to define our high confidence bounds. We use these to choose settings of the cutoff parameters $ \alpha_L, \alpha_U$. Finally, we define the good events whereby the terms from the decomposition presented in the preceding section are bounded in terms of the design matrix and $ \beta_k(\delta'), \nu_k(\delta'),$ and $\gamma_k(\delta')$.

\begin{definition}[Filtrations]\label{rlsvi:deffiltration}
    For any $ t \in [H]$ and any $ k$ define the filtrations
    \begin{align}
        \mathcal{H}_{tk} &\defeq \{s_{ij}, a_{ij}, r_{ij}: j \leq k,\quad i\leq t \text{ if } j=k \text{ else } i \leq H\}\\
        \mathcal{H}_{k} &\defeq \mathcal{H}_{H, k}\\
        \overline{\mathcal{H}}_{tk} &\defeq \mathcal{H}_k \ \bigcup\  \{\overline{\xi}_{ik}: i \geq t\}\\
        \overline{\mathcal{H}}_{k} &\defeq \overline{\mathcal{H}}_{1k}
    \end{align}
\end{definition}

\begin{definition}[Noise bounds]\label{rlsvi:defnoise_bounds}
    For some constants $ c_1, c_2 $ let
    \begin{align}
        \sqrt{\beta_k(\delta')} &\defeq c_1 Hd \sqrt{\log \(\frac{Hdk \max(1,L_\phi) \max(1,L_\psi) \max(1,L_r) \lambda}{\delta'}\)}\\
        \sqrt{\nu_k(\delta')} & \defeq \sqrt{\beta_k(\delta')}  + \sqrt{\lambda} L_\phi ( 3HL_\psi + L_r ) + 4 \epsilon H \sqrt{dk}\\
        \sqrt{\gamma_k(\delta')} &\defeq c_2\sqrt{dH \nu_{k}(\delta') \log(d/\delta')}
    \end{align}
\end{definition}
Note that these functions are monotonically increasing in $k$, e.g., $\sqrt{\beta_k(\delta')} \leq \sqrt{\beta_{k+1}(\delta')}$.

\begin{definition}[Default cutoff]\label{rlsvi:defdefault}
    Set
    \begin{align}
        \alpha_U &\defeq \frac{1}{4(\sqrt{\gamma_k(\delta')})} \leq \frac{1}{2(\sqrt{\nu_k(\delta')} + \sqrt{\gamma_k(\delta')})}\\
        \alpha_L &\defeq \alpha_U / 2
    \end{align}
\end{definition}

\begin{definition}[Good event]\label{rlsvi:defgood_event}
    Define
    \begin{align}
        \mathcal{G}^{\overline{\xi}}_{tk} &\defeq \bigg\{ |\phi_t(s,a)^\top \overline{\xi}_{tk}| \leq \sqrt{\gamma_k(\delta')} \|\phi_t(s,a)\|_{\Sigma_{tk}^{-1}}\bigg\} \\
        \mathcal{G}^{\overline \eta}_{tk} &\defeq \bigg\{ |\phi_t(s,a)^\top \overline \eta_{tk}| \leq \sqrt{\beta_k(\delta')} \|\phi_t(s,a)\|_{\Sigma_{tk}^{-1}}\bigg\}\\
        \mathcal{G}^{\overline \lambda}_{tk} &\defeq \bigg\{\forall \ \pi, \quad |\phi_t(s,a)^\top\overline \lambda^\pi_{tk}| \leq \sqrt{\lambda} L_\phi ( 3HL_\psi + L_r ) \|\phi_t(s,a)\|_{\Sigma_{tk}^{-1}}\bigg\}\\
        \mathcal{G}^{\overline m}_{tk} &\defeq \bigg\{\forall \ \pi, \quad |\overline m_{tk}^\pi (s,a)| \leq 4 \epsilon H (\sqrt{dk}\|\phi_t(s,a)\|_{ \Sigma_{tk}^{-1}} + 1)  \bigg\}\\
        \mathcal{G}^{\overline Q}_{tk} &\defeq \bigg\{\forall
        \ s,a,\quad |(\overline Q_{tk} - Q_t^\star)(s,a)| \leq H - t + 1\bigg\}
    \end{align}
    And then the good events are the intersections
    \begin{align}
        \overline{\mathcal{G}}_{tk} &\defeq \{  \mathcal{G}^{\overline{\xi}}_{tk} \cap \mathcal{G}^{\overline \eta}_{tk} \cap \mathcal{G}^{\overline \lambda}_{tk} \cap \mathcal{G}^{\overline m}_{tk} \cap \mathcal{G}^{\overline Q}_{tk}\}\\
        \overline{\mathcal{G}}_k &\defeq \bigcap_{ t \in [H]} \overline{\mathcal{G}}_{tk}
    \end{align}
\end{definition}

\newpage
\section{Concentration}\label{rlsvi:sec:concentration}
This section will prove that the good events happen with high probability. The tricky part is showing that the estimates $\overline{Q}_{tk}$ remain nicely bounded. To do this we bound each of the four separate terms ({\color{RawSienna}misspecification}, {\color{DarkOrchid}regularization}, {\color{ForestGreen}pseudonoise}, and {\color{BrickRed}environment noise}) with high probability when conditioned on bounded $ \overline Q$ values at time $ t+1$. Then we use an inductive argument to show that this means that all terms and the $ \overline{Q}$ values are bounded across all timesteps with high probability.

\subsection{Bounding the Misspecification Error}
\begin{lemma}[\textcolor{RawSienna}{Misspecification}]\label{rlsvi:lemmisspec_bound}
For any $ t,k,s,a $ and any policy $ \pi$, if
\begin{align}
    \bigg|(\overline Q_{t+1,k} - Q_{t+1}^\star)(s,a) \bigg| &\leq H - t
\end{align}
then
\begin{align}
    |\overline m_{tk}^\pi (s,a)| \leq 4 \epsilon H \(\sqrt{dk}\|\phi_t(s,a)\|_{\Sigma_{tk}^{-1}} + 1\)
\end{align}
\end{lemma}
\begin{proof}
Recall the definition of $\overline m_{tk}^\pi$ in Eq.~\ref{eq:def_mtkpi}
\begin{align}
    |\overline m_{tk}^\pi(s,a)| &= \bigg|\phi_t(s,a)^\top\Sigma_{tk}^{-1}\sum_{i=1}^{k-1} \phi_{ti} \bigg[\Delta_t^r(s_{ti}, a_{ti}) + \int_{s'}\Delta_t^P(s'|s_{ti}, a_{ti})\overline V_{t+1,k}(s') \bigg] + \Delta_{t}^\pi(s,a)\\&\qquad- \int_{s'} \Delta_t^P(s'|s,a)(\overline V_{t+1,k}- V_{t+1}^\pi)(s')\bigg|.
\end{align}
Under event $\mathcal{G}_{t+1.k}^{\overline{Q}}$, we have that $ |(\overline V_{t+1,k}- V_{t+1}^\pi)(s')| \leq |(\overline V_{t+1,k}- V_{t+1}^\star)(s')| + |(V_{t+1}^\star- V_{t+1}^\pi)(s')| \leq 2H $.
Then, applying the triangle inequality, Holder, and bounds from Definition \ref{rlsvi:defmisspec} and Corollary \ref{cor:linear_Q} as well as previous bound on the estimated value functions, we can erite
\begin{align}
    |\overline m^\pi_{tk}(s,a)|\leq (\epsilon + \epsilon H)\bigg|\phi_t(s,a)^\top\Sigma_{tk}^{-1}\sum_{i=1}^{k-1} \phi_{ti}\bigg| + \epsilon H + 3\epsilon H.
\end{align}
Finally, grouping terms and applying Cauchy-Schwarz twice we get
\begin{align}
    & \leq 4\epsilon H \bigg( \|\phi_t(s,a)\|_{\Sigma_{tk}^{-1}} \bigg\|\sum_{i=1}^{k-1} \phi_{ti} \bigg\|_{\Sigma_{tk}^{-1}} + 1\bigg)\\
    &\leq 4\epsilon H \bigg( \sqrt{k}\|\phi_t(s,a)\|_{\Sigma_{tk}^{-1}} \bigg(\sum_{i=1}^{k-1} \|\phi_{ti} \|_{\Sigma_{tk}^{-1}}^2\bigg)^{1/2} + 1\bigg).
\end{align}
The result follows by applying Lemma \ref{rlsvi:lemfeature_sum_final}.
\end{proof}

\subsection{Bounding the Regularization}

\begin{lemma}[\textcolor{DarkOrchid}{Regularization}]\label{rlsvi:lemregularization}
For any $ t,k, \pi$ and any features $ \phi_t(s,a)$, if
\begin{align}
    \bigg|(\overline Q_{t+1,k} - Q_{t+1}^\star)(s,a) \bigg| &\leq H - t
\end{align}
then
\begin{align}
    |\phi_t(s,a)^\top \overline \lambda_{tk}^\pi| \leq \sqrt{\lambda} L_\phi ( 3HL_\psi + L_r) \|\phi_t(s,a)\|_{\Sigma_{tk}^{-1}}
\end{align}
\end{lemma}
\begin{proof}
By Cauchy-Schwarz and the fact that the maximal eigenvalue of $ \Sigma_{tk}^{-1}$ is at most $ 1/\lambda$
\begin{align}
|\phi_t(s,a)^\top \overline \lambda_{tk}^\pi| &= \bigg| \phi_t(s,a)^\top \lambda \Sigma_{tk}^{-1}\bigg(\int_{s'} \psi_t(s') (\overline V_{t+1,k} - V_{t+1}^\pi)(s') + \theta_t^\pi \bigg)\bigg|\\
&\leq \sqrt{\lambda} \|\phi_t(s,a)\|_{\Sigma_{tk}^{-1}} \(\bigg\|\int_{s'} \psi_t(s') (\overline V_{t+1,k} - V_{t+1}^\pi)(s')\bigg\| + \| \theta_t^\pi\|\)\\
&\leq \sqrt{\lambda} \|\phi_t(s,a)\|_{\Sigma_{tk}^{-1}} \(\int_{s'} \bigg\|\psi_t(s') (\overline V_{t+1,k} - V_{t+1}^\pi)(s')\bigg\| + \| \theta_t^\pi\|\)\\
&\leq \sqrt{\lambda} \|\phi_t(s,a)\|_{\Sigma_{tk}^{-1}} \(\int_{s'} \|\psi_t(s') \| |(\overline V_{t+1,k} - V_{t+1}^\pi)(s')| + \| \theta_t^\pi\|\)\\
&\leq \sqrt{\lambda} \|\phi_t(s,a)\|_{\Sigma_{tk}^{-1}} \( \|\overline V_{t+1,k} - V_{t+1}^\pi\|_\infty\int_{s'} \|\psi_t(s') \|  + \| \theta_t^\pi\|\)
\end{align}
Applying the hypothesis of the lemma and the bounds from Assumption \ref{rlsvi:assapprox_lowrank} and Corollary \ref{cor:linear_Q}
\begin{align}
    \leq \sqrt{\lambda} L_\phi  [2H L_\psi + (L_r + (H-t) L_\psi))]\|\phi_t(s,a)\|_{\Sigma_{tk}^{-1}} \leq \sqrt{\lambda} L_\phi ( 3HL_\psi + L_r )\|\phi_t(s,a)\|_{\Sigma_{tk}^{-1}}.
\end{align}
\end{proof}

\subsection{Bounding the Environment Noise}

\begin{lemma}[Concentration inductive step]\label{rlsvi:lemconcentration_induction}
Fix $t$ and $k$. For any $ \delta' > 0$ and conditioned for all $ s,a$ and all $ z > t$ on
\begin{align}
\label{eqn:lemci_h1}
    \bigg|(\overline Q_{z,k} - Q_{z}^\star)(s,a) \bigg| &\leq H - t
\end{align}
and on
\begin{align}
    \|\overline\xi_{t+1,k}\|_{\Sigma_{t+1,k}} &\leq   \sqrt{\gamma_k(\delta')}
\end{align}
then with probability at least $ 1-\delta'$
\begin{align}
    | \phi_t(s,a)^\top \overline \eta_{tk} |  &\leq  \sqrt{\beta_k(\delta')} \| \phi_t(s,a) \|_{\Sigma_{tk}^{-1}}
\end{align}
\end{lemma}
\begin{proof}

Recall the definition of $\overline{\eta}_{tk}$ given in Eq.~\ref{eq:def_etatk}.
By Cauchy-Schwarz:
\begin{align}
    | \phi_t(s,a)^\top \overline \eta_{tk}
    | \leq \|\phi_t(s,a) \|_{\Sigma_{tk}^{-1}} \|\overline \eta_{tk}\|_{\Sigma_{tk}^{-1}}
\end{align}
where
\begin{align}
    \|\overline \eta_{tk}\|_{\Sigma_{tk}^{-1}} = \bigg\| \sum_{i=1}^{k-1} \phi_{ti}\bigg( \overline V_{t+1,k}(s_{t+1,i}) - \E_{s'|s_{ti}, a_{ti}}[\overline V_{t+1,k}(s')]\bigg) \bigg\|_{\Sigma_{tk}^{-1}}
\end{align}

First, we will show that given the hypothesis of the lemma, we can bound
\begin{align}
    \| \overline \theta_{t+1,k} \| \leq 2 H \sqrt{kd / \lambda} + \sqrt{\gamma_k(\delta')/\lambda}
\end{align}
To see this, note that $ \|\overline V_{t+2,k} \|_\infty \leq 2(H-t-1)$ from Eq.~\ref{eqn:lemci_h1} and so applying Cauchy-Schwarz gives us
\begin{align}
    \| \widehat \theta_{t+1,k} \| &= \|\Sigma_{t+1,k}^{-1} \sum_{i=1}^{k-1} \phi_{ti} (r_{t+1,i} + \overline V_{t+2,k}(s_{t+2,i}))\| \leq \|\Sigma_{t+1,k}^{-1/2}\| \|\sum_{i=1}^{k-1} \phi_{t+1,i} (r_{t+1,i} + \overline V_{t+2,k}(s_{t+2,i}))\|_{\Sigma_{t+1,k}^{-1}}\\
    &\leq \frac{1}{\sqrt{\lambda}} \sqrt{k} \bigg(\sum_{i=1}^{k-1} \|\phi_{t+1,i} (r_{t+1,i} + \overline V_{t+2,k}(s_{t+2,i}))\|^2_{\Sigma_{t+1,k}^{-1}}\bigg)^{1/2}\\
    &\leq \frac{1}{\sqrt{\lambda}} (2(H-t -1 ) + 1) \sqrt{k} \bigg(\sum_{i=1}^{k-1} \|\phi_{ti} \|^2_{\Sigma_{tk}^{-1}}\bigg)^{1/2}\\
    &\leq 2 H \sqrt{kd / \lambda}
\end{align}
where the last inequality comes from Lemma \ref{rlsvi:lemfeature_sum_final}. With this bound in hand, we can now proceed with a covering argument over the functions $ \overline V_{t+1,k}$ to bound $ \overline \eta_{tk}$.

For any $ \theta \in \R^d $ with $ \|\theta \| \leq 2 H \sqrt{kd / \lambda}  + \sqrt{\gamma_k(\delta')/\lambda}$ and $ \Sigma \in \R^{d\times d}$ symmetric and positive definite with $ \|\Sigma \| \leq \frac{1}{\lambda}$, we define
\begin{align} \label{eqn:defQthetasigma}
    Q_t^{\theta, \Sigma}(s,a) \defeq \begin{cases}
      \phi_t(s,a)^\top \theta, & \text{if}\ \| \phi_t(s,a)\|_{\Sigma} \leq \alpha_L\\
       H-t + 1 &  \text{if}\ \| \phi_t(s,a)\|_{\Sigma} \geq \alpha_U \\
      \(\frac{\alpha_U - \|\phi_t(s,a)\|_\Sigma }{\alpha_U - \alpha_L}\)\phi_t(s,a)^\top \theta + \(\frac{ \|\phi_t(s,a)\|_\Sigma - \alpha_L}{\alpha_U - \alpha_L}\)(H-t + 1) & \text{otherwise}
    \end{cases}
\end{align}

Let $ V^{\theta, \Sigma} $ be the corresponding value function. Note that $ \overline V_{t+1,k} = V^{\overline \theta_{t+1,k}, \Sigma_{t+1, k}^{-1}}$.

Define
\begin{align}\label{eqn:defotp1}
    O_{t+1} \defeq \bigg\{\theta, \Sigma: \|\theta \| \leq 2 H \sqrt{kd / \lambda} + \sqrt{\gamma_k(\delta')/\lambda},\quad \|\Sigma \| \leq \frac{1}{\lambda}, \\ \quad |(Q_{t+1}^{\theta, \Sigma} - Q_{t+1}^\star)(s,a)| \leq H - t \ \ \forall\ s,a\bigg\}
\end{align}
So that by the hypothesis of the lemma, $\overline \theta_{t+1,k}, \Sigma_{t+1, k}^{-1} \in O_{t+1}$.

For any $ (\theta, \Sigma) \in O_{t+1}$ and $ i \in [k-1]$ define
\begin{align}
    x^{\theta, \Sigma}_i \defeq V^{\theta, \Sigma} (s_{t+1,i}) - \E_{s'|s_{ti}, a_{ti}}[V^{\theta, \Sigma} (s')]
\end{align}
Then $ x_i $ defines a martingale difference sequence with filtration $ \mathcal{H}_{ti}$. Moreover, by the definition of $ O_{t+1}$, each $ x_i$ is bounded in absolute value by $ 2H$ (from last condition in~\eqref{eqn:defotp1}) so that each $ x_i$ is a $ 2H$-subgaussian random variable.

So, by Lemma \ref{rlsvi:lemself_normalized} the $ x^{\theta, \Sigma}_i$ induce a self normalizing process so that
\begin{align}
    \bigg\| \sum_{i=1}^{k-1} \phi_i x^{\theta, \Sigma}_i \bigg\|_{\Sigma_{tk}^{-1}}  \leq 4H\bigg(d\log\(\frac{k L_\phi^2 + \lambda}{\lambda}\) + \log(1/\delta')\bigg)^{1/2}
\end{align}

Note that the $\varepsilon$-covering number of $ O_{t+1}$ as a Euclidean ball in $ \R^{d+d^2}$ of radius $ 2H\sqrt{kd/\lambda} + \sqrt{\gamma_k(\delta')/\lambda} + 1/\lambda$, denoted $N_\varepsilon(O_{t+1})$, is bounded by Lemma \ref{rlsvi:lemcovering_number} as $ (3 (2H\sqrt{kd/\lambda} + \sqrt{\gamma_k(\delta')/\lambda} +  1/\lambda)/\varepsilon)^{d^2 + d}$.
So, by a union bound, with probability at least $ 1- \delta' $ we have for all $ (\theta, \Sigma) \in O_{t+1}$ that
\begin{align}
    \bigg\| \sum_{i=1}^{k-1} \phi_i x^{\theta, \Sigma}_i \bigg\|_{\Sigma_{tk}^{-1}}  &\leq 4H\bigg(d\log\(\frac{k L_\phi^2 + \lambda}{\lambda}\) + \log(N_\varepsilon(O_{t+1}) /\delta')\bigg)^{1/2} \\&\leq  4H\bigg(d\log\(\frac{k L_\phi^2 + \lambda}{\lambda}\) \\ & + (d^2 + d) \log \bigg(3(2H\sqrt{kd/\lambda} + \sqrt{\gamma_k(\delta')/\lambda}+ 1/\lambda)/\varepsilon\bigg)+ \log(1 /\delta')\bigg)^{1/2}\\\label{eqn:env_noise_bound_long}
    &\leq 8Hd \bigg(\log\(\frac{k L_\phi^2 + \lambda}{\lambda}\) \\
    & + \log \bigg(3 (2H\sqrt{kd/\lambda} + \sqrt{\gamma_k(\delta')/\lambda}+ 1/\lambda)/\varepsilon\bigg) + \log(1 /\delta')\bigg)^{1/2}
\end{align}

To conclude the proof, we choose a specific $ (\theta, \Sigma)  \in O_{t+1}$ such that $ \|\theta - \overline \theta_{t+1,k}\| \leq \varepsilon$ and $ \|\Sigma - \Sigma_{t+1, k}^{-1}\|_F \leq \varepsilon$. Then
\begin{align}
     \|\eta_{tk}\|_{\Sigma_{tk}^{-1}}&= \bigg\| \sum_{i=1}^{k-1} \phi_i x^{\overline \theta_{t+1,k}, \Sigma_{t+1, k}^{-1}}_i \bigg\|_{\Sigma_{tk}^{-1}} \\ \label{eqn:env_noise_decomp}
     &\leq \bigg\| \sum_{i=1}^{k-1} \phi_i x^{\theta, \Sigma}_i \bigg\|_{\Sigma_{tk}^{-1}} + \bigg\| \sum_{i=1}^{k-1} \phi_i (x^{\theta, \Sigma}_i - x^{\overline \theta_{t+1,k}, \Sigma_{t+1, k}^{-1}}_i ) \bigg\|_{\Sigma_{tk}^{-1}}
\end{align}
Then we can bound
\begin{align}
     \bigg\| \sum_{i=1}^{k-1} \phi_i (x^{\theta, \Sigma}_i - x^{\overline \theta_{t+1,k}, \Sigma_{t+1, k}^{-1}}_i ) \bigg\|_{\Sigma_{tk}^{-1}} \leq kL_\phi \sup_i \bigg|x^{\theta, \Sigma}_i - x^{\overline \theta_{t+1,k}, \Sigma_{t+1, k}^{-1}}_i \bigg|
\end{align}

Plugging in the definition of the $ x_i$ and applying Lemma \ref{rlsvi:lemcovering_lemma} we bound
\begin{align}
    \sup_i \bigg|x^{\theta, \Sigma}_i - x^{\overline \theta_{t+1,k}, \Sigma_{t+1, k}^{-1}}_i \bigg| &= \sup_i \bigg|(V^{\theta, \Sigma} - \overline V_{t+1,i}) (s_{t+1,i}) - \E_{s'|s_{ti}, a_{ti}}[(V^{\theta, \Sigma} - \overline V_{t+1,i})(s')] \bigg|\\
    &\leq 2 \sup_{s,a} |(Q^{\theta, \Sigma} - \overline Q_{t+1,k})(s,a)|\\
    &\leq 2 \sqrt{\varepsilon}\frac{ L_\phi (4H^2)}{\alpha_U - \alpha_L}
\end{align}

So we can bound the covering error by 1 if we choose $\varepsilon$ small enough such that
\begin{align}
    \varepsilon \leq \(\frac{\alpha_U - \alpha_L}{8kL_\phi^2 H^2}\)^2
\end{align}

Then with probability at least $ 1-\delta$, combining (\ref{eqn:env_noise_bound_long}) with (\ref{eqn:env_noise_decomp}),  \ref{rlsvi:lemlog_term}, and the choice of $ \varepsilon$ we get
\begin{align}
    \|\overline\eta_{tk}\|_{\Sigma_{tk}^{-1}}
     &\leq \bigg\| \sum_{i=1}^{k-1} \phi_i x^{\theta, \Sigma}_i \bigg\|_{\Sigma_{tk}^{-1}} + 1 \leq \sqrt{\beta_k(\delta')}
\end{align}
as desired.
\end{proof}

\begin{lemma}[Covering Lemma]\label{rlsvi:lemcovering_lemma}
This lemma uses the notation defined within the previous lemma, suppressing indices. Take $ (\theta, \Sigma)$ and $(\theta', \Sigma')$ in $ O $ (see Eq.~\ref{eqn:defotp1} for generic $ t$) such that $\|\theta - \theta'\| \leq \varepsilon$ and $ \|\Sigma - \Sigma'\| \leq \varepsilon$ with $ \varepsilon\leq \min \{1, \frac{H}{3L_\phi}, \frac{\alpha_U - \alpha_L}{L_\phi^2}\}$, then
\begin{align}
    \sup_{s,a} |(Q^{\theta, \Sigma} - Q^{\theta', \Sigma'})(s,a)|  \leq \sqrt{\varepsilon}\frac{ L_\phi (4H^2)}{\alpha_U - \alpha_L}
\end{align}
\end{lemma}
\begin{proof}
Note that by the assumption, for any $ \phi$ with $ \|\phi\|\leq L_\phi$
\begin{align}\label{eqn:matrix_net}
    |\|\phi\|_\Sigma - \|\phi\|_{\Sigma'}| &= \bigg|\sqrt{\phi^\top \Sigma \phi} - \sqrt{\phi^\top \Sigma' \phi}\bigg| \leq \sqrt{|\phi^\top (\Sigma - \Sigma')\phi|} \leq \sqrt{\|\phi\| \|(\Sigma - \Sigma')\|\|\phi\|} \leq \sqrt{\varepsilon} L_\phi
\end{align}

Now we need to split into cases. Since $ \theta, \Sigma $ and $ \theta', \Sigma'$ are interchangable, the following 5 cases cover all possibilities.

\noindent\textbf{Case 1 (linear-linear):} $ \|\phi(s,a)\|_{\Sigma} \leq \alpha_L $ and $ \|\phi(s,a)\|_{\Sigma'}\leq \alpha_L$.

We can apply Cauchy-Schwarz and the definition of the case to get
\begin{align}\label{eqn:linearlinear}
    |(Q^{\theta, \Sigma} - Q^{\theta', \Sigma'})(s,a)| = |\phi(s,a)^\top (\theta - \theta')| \leq L_\phi \varepsilon
\end{align}

\noindent\textbf{Case 2 (linear-interpolating):} $ \|\phi(s,a)\|_{\Sigma} \leq \alpha_L $ and $ \alpha_L  \leq  \|\phi(s,a)\|_{\Sigma'} \leq \alpha_L + \sqrt{\varepsilon} L_\phi  \leq \alpha_U $

Applying (\ref{eqn:matrix_net}) and the definition of the case,
\begin{align}
    \|\phi(s,a)\|_{\Sigma'}\leq\|\phi(s,a)\|_{\Sigma} + |\|\phi(s,a)\|_{\Sigma'} - \|\phi(s,a)\|_{\Sigma}| \leq \alpha_L + \sqrt{\varepsilon} L_\phi.
\end{align}
Moreover, by our choice of $ \theta, \Sigma \in O$ which induces bounded Q functions we can bound
\begin{align}\label{eqn:cover_bounded_difference}
    |\phi(s,a)^\top \theta - (H-t) | \leq |\phi(s,a)^\top \theta| + H \leq 3H
\end{align}

So if we set
\begin{align}
    q' \defeq \frac{\|\phi(s,a)\|_{\Sigma'} - \alpha_L}{\alpha_U - \alpha_L} \leq \frac{\alpha_L + \sqrt{\varepsilon} L_\phi - \alpha_L}{\alpha_U - \alpha_L} = \frac{\sqrt{\varepsilon} L_\phi }{\alpha_U - \alpha_L}
\end{align}
then we have by the triangle inequality, equation (\ref{eqn:linearlinear}), and the above reasoning,
\begin{align}
    |(Q^{\theta, \Sigma} - Q^{\theta', \Sigma'})(s,a)| &= |\phi(s,a)^\top \theta - (1-q') \phi(s,a)^\top \theta' -  q'(H-t)|\\
     &\leq (1-q')  |\phi(s,a)^\top (\theta -\theta')|  +  q'| \phi(s,a)^\top \theta - (H-t)|\\\
    &\leq (1-q') L_\phi \varepsilon + q' |\phi(s,a)^\top \theta - (H-t) |\\
    &\leq L_\phi \varepsilon + \frac{\sqrt{\varepsilon} L_\phi (3H)}{\alpha_U - \alpha_L}
\end{align}

\noindent\textbf{Case 3 (default-default):} $ \alpha_U \leq \|\phi(s,a)\|_{\Sigma} $ and  $ \alpha_U  \leq \|\phi(s,a)\|_{\Sigma'}$.

Then we have that
\begin{align}
    |(Q^{\theta, \Sigma} - Q^{\theta', \Sigma'})(s,a)| = |(H-t) - (H-t)| = 0.
\end{align}

\noindent\textbf{Case 4 (default-interpolating):} $ \alpha_U \leq \|\phi(s,a)\|_{\Sigma} $ and $ \alpha_L  \leq \alpha_U - \sqrt{\varepsilon} L_\phi \leq  \|\phi(s,a)\|_{\Sigma'}  \leq \alpha_U $

By the definition of the case
\begin{align}
      -\sqrt{\varepsilon} L_\phi \leq  \|\phi(s,a)\|_{\Sigma'} - \alpha_U,
\end{align}
so that defining $ q' $ as before
\begin{align}\label{eqn:qprimebound}
    1-q' = 1 - \frac{\|\phi(s,a)\|_{\Sigma'} - \alpha_L}{\alpha_U - \alpha_L} = \frac{\alpha_U - \|\phi(s,a)\|_{\Sigma'} }{\alpha_U - \alpha_L} \leq  \frac{\sqrt{\varepsilon} L_\phi}{\alpha_U - \alpha_L}.
\end{align}
And thus, applying (\ref{eqn:qprimebound}) and (\ref{eqn:cover_bounded_difference}) again we get
\begin{align}
    |(Q^{\theta, \Sigma} - Q^{\theta', \Sigma'})(s,a)| &= |(H-t) - (1-q')\phi(s,a)^\top \theta' - q' (H-t)|\\
    &\leq (1-q') |\phi(s,a)^\top \theta' -  (H-t)|\\
    &\leq \frac{\sqrt{\varepsilon} L_\phi(3H)}{\alpha_U - \alpha_L}
\end{align}

\noindent\textbf{Case 5 (interpolating-interpolating):} $\alpha_L  \leq  \|\phi(s,a)\|_{\Sigma'}  \leq \alpha_U$ and $ \alpha_L  \leq  \|\phi(s,a)\|_{\Sigma'}  \leq \alpha_U $

Letting $ q $ be analogous to $ q'$ but for $ \Sigma$ and applying (\ref{eqn:matrix_net}) we have
\begin{align}
    |q - q'| = \frac{|\|\phi(s,a)\|_{\Sigma} - \alpha_L - (\|\phi(s,a)\|_{\Sigma'} - \alpha_L)|}{\alpha_U - \alpha_L} \leq \frac{\sqrt{\varepsilon}L_\phi}{\alpha_U - \alpha_L}
\end{align}
Thus we have that
\begin{align}
    |(Q^{\theta, \Sigma} - Q^{\theta', \Sigma'})(s,a)| &= | (1-q)\phi(s,a)^\top \theta + q (H-t)  \\&\qquad\qquad- (1-q')\phi(s,a)^\top \theta' - q' (H-t)|\\
    &\leq \frac{\sqrt{\varepsilon} L_\phi(3H)}{\alpha_U - \alpha_L} \(|\phi(s,a)^\top (\theta - \theta')| + H\) \\
    &\leq \frac{\sqrt{\varepsilon} L_\phi(3H)}{\alpha_U - \alpha_L} \(L_\phi \varepsilon  + H\) \leq \frac{\sqrt{\varepsilon} L_\phi(4H^2)}{\alpha_U - \alpha_L}
\end{align}
Taking the max over all of the cases (which is case 5) yields the result.
\end{proof}

\subsection{Bounding the Q values}

\begin{lemma}[Boundedness inductive step]\label{rlsvi:lembounded_induction}
Assume that $ \epsilon < \frac{1}{10H}$ and that for all $ s,a$
\begin{align}
    \bigg|(\overline Q_{t+1,k} - Q_{t+1}^\star)(s,a) \bigg| &\leq H - t\\
    \label{eqn:bounded_induction_inductive_property}
    \bigg| \phi_t(s,a)^\top (\overline \eta_{tk} + \overline \xi_{tk} + \overline \lambda_{tk}^\star) + \overline m_{tk}^\star(s,a)
    \bigg|  &\leq  (\sqrt{\nu_k(\delta')} + \sqrt{\gamma_k(\delta')})\| \phi_t(s,a) \|_{\Sigma_{tk}^{-1}} + 4 \epsilon H
\end{align}
where $\overline \lambda_{tk}^\star$ and $\overline m_{tk}^\star$ are as in~\eqref{eq:def_lambdatkpi} and~\eqref{eq:def_mtkpi} with $\pi=\pi^\star$, then for all $ s,a $
\begin{align}
    \bigg|(\overline Q_{tk} - Q_{t}^\star)(s,a) \bigg| &\leq H - t + 1
\end{align}
\end{lemma}

\begin{proof}
There are two cases, depending on whether the features are large.

\noindent\textbf{Case 1 (large features):} $ \|\phi_t(s,a) \|_{\Sigma_{tk}^{-1}} \geq \alpha_U$.

Then by the definition of $ \overline{Q}_{tk}$ from the algorithm (see~\eqref{eqn:defQthetasigma} or~Definition \ref{rlsvi:defQdef}), we have $ 0 \leq \overline{Q}_{tk}(s,a) \leq H - t + 1$. Since $ Q_t^\star$ must be in the same range, we immediately get
\begin{align}
    \bigg|(\overline Q_{tk} - Q_{t}^\star)(s,a) \bigg| &\leq H - t + 1
\end{align}

\noindent\textbf{Case 2 (small features):} $\|\phi_t(s,a) \|_{\Sigma_{tk}^{-1}} \leq \alpha_L$.

In this case we get $ \overline Q_{tk}(s,a) = \phi_t(s,a)^\top \overline{\theta}_{tk}$. So we apply Lemma \ref{rlsvi:lemQ_decomp} to get
\begin{align}
    \bigg|(\overline Q_{tk} - Q_{t}^\star)(s,a) \bigg| &= \bigg| \E_{s'|s,a}[\( \overline V_{t+1,\pk} - V^\star_{t+1} \)(s')]  + \phi_t(s,a)^\top (\overline \eta_{tk} + \overline\xi_{tk} + \overline \lambda_{tk}^\star) + \overline m_{tk}^\star(s,a) \bigg|.
\end{align}
We can split the terms by the triangle inequality. Using the inductive hypothesis~\eqref{eqn:bounded_induction_inductive_property} gives us
\begin{align}
    &\leq H - t + (\sqrt{\nu_k(\delta')} + \sqrt{\gamma_k(\delta')})\underbrace{\| \phi_t(s,a) \|_{\Sigma_{tk}^{-1}}}_{\leq \alpha_L \text{ since Case 2 holds}} + 4\epsilon H.
\end{align}
Finally, by our choice of $ \alpha_L $ (see Definition~\ref{rlsvi:defdefault})  and using $ \epsilon < \frac{1}{10H}$ we get the final bound
\begin{align}
    \leq H - t + 1.
\end{align}

\noindent\textbf{Case 3 (medium features):} $\alpha_L  \leq \|\phi_t(s,a) \|_{\Sigma_{tk}^{-1}} \leq \alpha_U $.

This case immediately follows from applying the first two cases and our choice of $ \alpha_U$ (see Definition~\ref{rlsvi:defdefault}) along with noting that for any $ Q^1, Q^2$
\begin{align}
    |q Q^1(s,a) + (1-q) Q^2(s,a)  - Q_t^\star(s,a)| \leq  q| (Q^1 - Q_t^\star)(s,a)|  + (1-q)| (Q^2- Q_t^\star)(s,a)|
\end{align}
So that when both $ Q^1, Q^2 $ satisfy the desired relationship to $ Q^\star$, so does their interpolation.
\end{proof}

\subsection{Putting it All Together: Good Event with High Probability}\label{rlsvi:sec:good_event}

\begin{lemma}[Good event probability]\label{rlsvi:lemgood_prob}
For $ \epsilon < \frac{1}{10H}$, with probability at least $ 1-\delta/8$ we have $ \bigcap_{k \leq K} \overline{\mathcal{G}}_k$.
\end{lemma}

\begin{proof}
For each $ k $ we will induct backwards over $ t$ using the preceding lemmas to prove that $ \overline{\mathcal{G}}_{tk}$ occurs for all $ t \in [H]$ with probability at least $ 1 - \delta/(8K)$. In the following, recall that $ \delta' = \delta / (16HK)$.

As the base case, consider step $ H$.  Since we define $ \overline Q_{H+1, k} = 0 = Q_{H+1}^\star$, we can invoke Lemmas \ref{rlsvi:lemmisspec_bound} and \ref{rlsvi:lemregularization} to get $ \mathcal{G}^{\overline \lambda}_{Hk}$ and $ \mathcal{G}^{\overline m}_{Hk}$. Then we can apply Lemma \ref{rlsvi:lemgaussian_concentration} so that  and $ \mathcal{G}^{\overline\xi}_{Hk}$ occurs with probability $ 1- \delta'$. Then we can invoke Lemma \ref{rlsvi:lemconcentration_induction} to get that conditioned on all these other events we get $ \mathcal{G}^{\overline \eta}_{Hk}$ with probability at least $ 1- \delta'$. Thus, we get the intersection of these events $ \{  \mathcal{G}^{\overline{\xi}}_{Hk} \cap \mathcal{G}^{\overline \eta}_{Hk} \cap \mathcal{G}^{\overline \lambda}_{Hk} \cap \mathcal{G}^{\overline m}_{Hk}  \}$ with probability at least $ (1-\delta')^2$. Finally, conditioned on $ \{ \mathcal{G}^{\overline Q}_{H+1,k} \cap \mathcal{G}^{\overline{\xi}}_{Hk} \cap \mathcal{G}^{\overline \eta}_{Hk} \cap \mathcal{G}^{\overline \lambda}_{Hk} \cap \mathcal{G}^{\overline m}_{Hk}  \}$ we can invoke Lemma \ref{rlsvi:lembounded_induction} (using the condition on $ \epsilon$) to get $\mathcal{G}^{\overline Q}_{Hk}$. Combining, we see that $ P(\overline{\mathcal{G}}_{Hk}) \geq (1 - \delta')^2$. The inductive step follows the same outline so that conditioning on $ \overline{\mathcal{G}}_{tk}$ we have $ P(\overline{\mathcal{G}}_{t-1,k}| \overline{\mathcal{G}}_{tk}) \geq (1-\delta')^2$. Thus, we can bound
\begin{align}
    P\( \overline{\mathcal{G}}_{k}\) \geq (1 - \delta')^{2H} \geq 1 - 2H \delta' = 1 - \delta/ (8K)
\end{align}
A union bound over $ k \in [K]$ gives the result.
\end{proof}

\newpage
\section{Optimism}\label{rlsvi:sec:optimism}

In this section we discuss how Algorithm \ref{alg:RLSVI} can ensure optimism, and we use $\widetilde \xi$ instead of $\overline\xi$ to indicate the pseudonoise. This facilitates the proof of Lemma \ref{rlsvi:lemPessimism} later, but \emph{the reader should think of the  $\widetilde \xi$'s as independent and identically distributed copies of the $\overline \xi$'s} (therefore with the same `properties').

To discuss optimism, in Lemma \ref{rlsvi:lemOptimisticRecursion} we lower bound the value function difference by a one-dimensional random walk. The idea is to look at the probability that the algorithm is optimistic along the optimal policy $\pistar$. If condition $\Smallcondition{tk}{\phi_t(x_t,\pistar_t(x_t))}$  was true at every $x_t$ encountered upon following the optimal policy $\pistar$, the random variable in the random walk that we obtain would be the projection of the pseudonoise $\overline \xi$ along the average feature $\phi$ encountered upon following $\pistar$. In fact, since $\Smallcondition{tk}{\phi_t(x_t,\pistar_t(x_t))}$ does not always hold, we end up not projecting on the average $\phi_t$ but on a different $\phi^\star_t$; importantly, this $\phi^\star_t$ is a \emph{non-random} quantity when conditioned on the history $\H_k$ and starting state $s_{1k}$. This allows us to show optimism by looking at properties of a normal random walk in lemma \ref{rlsvi:lemOptimisticRecursion}.

\begin{lem}[Optimistic Recursion]
Condition on the starting state $s_{1k}$, the history $\widetilde \H_k$ (which is $ \overline\H_k $ with $ \widetilde \xi_{tk}$ in place of $ \overline\xi_{tk}$), and the good event $ \widetilde{\mathcal{G}}_k$ (again with $ \widetilde \xi_{tk}$ in place of $ \overline\xi_{tk}$). Then for every timestep $t \in [H]$ there exists vector $\phi^\star_t \in \R^d$ that does not depend on any $ \widetilde\xi_{tk} $ such that:
\label{rlsvi:lemOptimisticRecursion}
\begin{align}
\( \widetilde V_{1} - \Vstar_{1}\)(s_{1k})  & \geq  \sum_{t=1}^H [(\phi^\star_t)^\top \widetilde \xi_{t\pk} - \sqrt{\nu_{ k}(\delta')}\|\phi_t^\star \|_{\Sigma_{t k}^{-1}}] - 4 H^2\epsilon .
\end{align}
\end{lem}

\begin{proof}
The proof proceed by induction, and is split into sections.

We will use $ x $ rather than $ s $ to emphasize the difference between states sampled with $ \pi^\star$ (denoted by $ x$) from those sampled with our policy $ \pi_k$ (denoted by $ s$).
Before to proceed, recall the definition of $\widetilde{Q}$ (same as $\overline{Q}$) from~\eqref{eqn:defQthetasigma} or~\eqref{eqn:defqbar}.

\paragraph{Definitions.}
Recursively define the following functions $w_t : \StateSpace \rightarrow \R $ and $\mathring w_t : \StateSpace \rightarrow \R $, which will be used to define $\phi^\star_t$:
\begin{align}
\label{eqn:defwtp1xtp1}
w_{t+1}(x_{t+1}) & = \int_{\StateSpace}\mathring w_t(x_t)\mathbb{P}_t(x_{t+1}|x_t,\pistar_t(x_t))dx_t \\
\mathring w_t(x_t) & = \begin{cases}
      w_t(x_t), & \text{if}\ \Smallcondition{tk}{\phi_{t}(s_t,\pistar_t(s_t))} \\
     \frac{\Condition{tk}{\phi_{t}(x_t,\pistar_t(x_t))} - \alpha_L}{\alpha_U-\alpha_L} w_t(x_t) , & \text{if }\Mediumcondition{tk}{\phi_{t}(x_t,\pistar_t(x_t))} \\
     0, & \text{if}\ \Largecondition{tk}{\phi_{t}(x_t,\pistar_t(x_t))}
    \end{cases} \\
w_1(x_1) & = 1 \\
x_1 &= s_{1k}
\label{eqn:wdef}
\end{align}

Then we can define
\begin{align}
    \phi^\star_t \defeq \int_{\StateSpace} \mathring w_{t}(x_{t}) \phi_t(x_t,\pistar_{t}(x_t)) dx_t.
\end{align}
Importantly, this choice of $ \phi_t^\star$ has no dependence on any $ \widetilde\xi_{tk}$ with $ t \in [H]$.

First we prove by induction that the $w_t$'s are positive and integrate to less than 1 for all $t \in [H]$:
\begin{align*}
w_t(x_t) & \geq 0, \; \forall x_t \in \StateSpace \\
\int_{S} w_t(x_t)dx_t & \leq 1
\numberthis{\label{eqn:wprop}}
\end{align*}
Positivity is immediate from the definition of equation (\ref{eqn:wdef}) since all quantities are positive. For the integral, assume by induction that at step $t$ it holds that $\int_{S} w_t(x_t)dx_t \leq 1$. For $t+1$ we have:

\begin{align}
\int_{\StateSpace} w_{t+1}(x_{t+1}) dx_{t+1} & = \int_{\StateSpace} \( \int_{\StateSpace}\mathring w_t(x_t)\mathbb{P}_t(x_{t+1}|x_t,\pistar_t(x_t))dx_t\)dx_{t+1} \\
& \stackrel{F}{=}\int_{\StateSpace} \( \int_{\StateSpace}\mathring w_t(x_t)\mathbb{P}_t(x_{t+1}|x_t,\pistar_t(x_t))dx_{t+1}\)dx_{t} \\
& = \int_{\StateSpace} \mathring w_t(x_t) \underbrace{\( \int_{\StateSpace}\mathbb{P}_t(x_{t+1}|x_t,\pistar_t(x_t))dx_{t+1}\)}_{=1}dx_{t} \\
& \leq \int_{\StateSpace} w_t(x_t) dx_{t} \leq 1.
\end{align}
In the last equality we used that $\mathring w_t \leq w_t$ pointwise (this follows directly by the definition), while step $F$ is due to Fubini's theorem for changing the order of integration.

\paragraph{Starting the main recursion.}
Let $\mathcal L_t$, $\mathcal M_t$ $\mathcal S_t$ be the event that the norm of the feature evaluated at $x_t$ and the optimal policy is large and small, respectively ($x_t$ is the random variable):
\begin{align}
	\mathcal S_t & \defeq \Big\{x_t: \Smallcondition{tk}{ \phi_t(x_t,\pistar_{t}(x_t)}\Big\} \\
		\mathcal M_t & \defeq \Big\{x_t: \Mediumcondition{tk}{\phi_t(x_t,\pistar_{t}(x_t)}\Big\} \\
	\mathcal L_t & \defeq \Big\{x_t: \Largecondition{tk}{\phi_t(x_t,\pistar_{t}(x_t)}\Big\}
\end{align}

First consider integrating over the state space with respect to $w_t(\cdot)$ the value function difference over the trajectories at step $t$ (the lower bound below holds for every term inside the expectation because $\pistar$ is the optimal policy on $\Qstar$ but not necessarily on $\widetilde Q$):
\begin{align}
& \int_{\StateSpace} w_t(x_t) \( \widetilde V_{t} - \Vstar_{t}\)(x_t) d x_t  \geq  \int_{\StateSpace} w_t(x_t) \( \widetilde Q_{t}(x_t,\pistar_t(x_t)) - \Qstar_{t}(x_t,\pistar_t(x_t)) \)dx_t
\end{align}
and then partition the statespace $\StateSpace$:

\begin{align}
& = \underbrace{ \int_{\mathcal S_t} w_t(x_t) \( \widetilde Q_{t} - \Qstar_{t} \)(x_t,\pistar_t(x_t)) dx_t}_{S} +  \underbrace{ \int_{\mathcal M_t} w_t(x_t) \( \widetilde Q_{t} - \Qstar_{t} \) (x_t,\pistar_t(x_t)) dx_t}_{M} \\ &\qquad+ \underbrace{ \int_{\mathcal L_t} w_t(x_t) \( \widetilde Q_{t} - \Qstar_{t} \) (x_t,\pistar_t(x_t)) dx_t}_{L}.
 \end{align}

We analyze each term individually.
\paragraph{Bound on the $L$ term.} Whenever $x_t \in \mathcal L_t$, Corollary \ref{cor:linear_Q} bounds the misspecification error so that:
 \begin{align}
 	L &= \int_{\mathcal L_t} w_t(x_t) \(H-t+1 - \Qstar_{t}  (x_t,\pistar_t(x_t))\) dx_t \\&\geq \int_{\mathcal L_t} w_t(x_t)[-(H-t+1)\epsilon]dx_t \geq -4H \epsilon \int_{\mathcal L_t} w_t(x_t)dx_t
 \end{align}
\paragraph{Bound on the $S$ term.} In states where the Q function is linear, the decomposition from
Lemma \ref{rlsvi:lemQ_decomp} gives us:
 \begin{align}
  & S = \int_{\mathcal S_t} w_t(x_t) \Biggm\{ \E_{x'|x_t,\pistar_t(x_t)}[ (\widetilde V_{t+1,\pk} - \Vstar_{t+1})(x') ] +  \phi_t(x_t,\pistar_{t}(x_t))^\top \( \widetilde\eta_{t\pk} + \widetilde \xi_{t\pk} +\widetilde\lambda^\star_{t\pk}\) + \widetilde m_{tk}^\star(x_t, \pi^\star_t(x_t))  \Biggm\} dx_t\\
  &\geq \int_{\mathcal S_t} w_t(x_t) \Biggm\{ \E_{x'|x_t,\pistar_t(x_t)}[ (\widetilde V_{t+1,\pk} - \Vstar_{t+1})(x') ] +  \phi_t(x_t,\pistar_{t}(x_t))^\top \widetilde \xi_{t\pk}  \Biggm\} dx_t \\ &\qquad\qquad +  \int_{\mathcal S_t} w_t(x_t) \Biggm\{ \phi_t(x_t,\pistar_{t}(x_t))^\top\(\widetilde\eta_{t\pk} + \widetilde\lambda^\star_{t\pk} +  \widetilde m_{tk, \phi}^\star \) - 4 H \epsilon \Biggm\} dx_t
 \end{align}
Where we introduce the new notation $ \widetilde m_{tk, \phi}^\star$ to indicate the portion of the misspecification term that depends on  the features $ \phi$. Explicitly, only taking the terms that multiply $ \phi $ from the definition of $ \widetilde m_{tk}^\pi$ in Eq. (\ref{eq:def_mtkpi}), we get that
\begin{align}
    \widetilde m_{tk, \phi}^\star \defeq \Sigma_{tk}^{-1}\sum_{i=1}^{k-1} \phi_{ti} \bigg[\Delta_t^r(s_{ti}, a_{ti}) + \int_{s'}\Delta_t^P(s'|s_{ti}, a_{ti})\overline V_{t+1,k}(s') \bigg].
\end{align}
This lets us split the two terms from the right hand side of the bound in Lemma \ref{rlsvi:lemmisspec_bound} so that the $ 4H\epsilon$ that does not depend on $ \phi$ can be introduced here.


\paragraph{Bound on the $M$ term.} This term interpolates between the values we would get out of the linearity of the representation and the default values.
Define $q^1$ and $q^2$ to be the coefficient of the linear interpolation (see~\eqref{eqn:defQthetasigma}), then:
\begin{align}
    M
    &= \int_{\mathcal M_t} w_t(x_t) \left( q^1 \widetilde{Q}_t(x_t, \pi_t^\star(x_t)) + q^2 (H-t+1) -Q^\star(x_t, \pi_t^\star(x_t))\right) \mathrm{d}x_t\\
    &= \int_{\mathcal M_t} w_t(x_t) \left( q^1
    \underbrace{
    (\widetilde{Q}_t-Q^\star)(x_t, \pi_t^\star(x_t))
    }_{\text{as in } S}
    + q^2
    \underbrace{
    ((H-t+1) -Q^\star(x_t, \pi_t^\star(x_t)))
    }_{\text{as in } L}
    +\underbrace{(q^1 + q^2 - 1)}_{=0}Q^\star(x_t, \pi_t^\star(x_t))\right) \mathrm{d}x_t\\
   &\geq \int_{\mathcal M_t} w_t(x_t) \frac{\Condition{tk}{\phi_{t}(x_t,\pistar(x_t))}-\alpha_L}{\alpha_U-\alpha_L} \Biggm\{ \E_{x'|x_t,\pistar_t(x_t)}[ (\widetilde V_{t+1,\pk} - \Vstar_{t+1})(x') ] + \phi_t(x_t,\pistar_{t}(x_t))^\top \widetilde \xi_{t\pk} \Biggm\} dx_t \\&\qquad\qquad +   \int_{\mathcal M_t} w_t(x_t) \frac{\Condition{tk}{\phi_{t}(x_t,\pistar(x_t))}-\alpha_L}{\alpha_U-\alpha_L} \Biggm\{ \phi_t(x_t,\pistar_{t}(x_t))^\top\(\widetilde\eta_{t\pk} + \widetilde\lambda^\star_{t\pk} + \widetilde m_{tk, \phi}^\star \) - 4H\epsilon \Biggm\} dx_t
   \\&\qquad\qquad - 4H\epsilon \int_{\mathcal M_t} w_t(x_t) \frac{\alpha_U - \Condition{tk}{\phi_{t}(x_t,\pistar(x_t))}}{\alpha_U-\alpha_L} dx_t\\
   & \geq \int_{\mathcal M_t} w_t(x_t) \frac{\Condition{tk}{\phi_{t}(x_t,\pistar(x_t))}-\alpha_L}{\alpha_U-\alpha_L} \Biggm\{ \E_{x'|x_t,\pistar_t(x_t)}[ (\widetilde V_{t+1,\pk} - \Vstar_{t+1})(x') ] + \phi_t(x_t,\pistar_{t}(x_t))^\top \widetilde \xi_{t\pk} \Biggm\} dx_t \\&\qquad\qquad +   \int_{\mathcal M_t} w_t(x_t) \frac{\Condition{tk}{\phi_{t}(x_t,\pistar(x_t))}-\alpha_L}{\alpha_U-\alpha_L} \Biggm\{ \phi_t(x_t,\pistar_{t}(x_t))^\top\(\widetilde\eta_{t\pk} + \widetilde\lambda^\star_{t\pk} + \widetilde m_{tk, \phi}^\star \) \Biggm\} dx_t
   \\&\qquad\qquad - 4H\epsilon \int_{\mathcal M_t} w_t(x_t) dx_t
  \end{align}

\paragraph{Conclusion.} Together, the bounds on $S,M,L$ we have obtained can be combined (also with the definition of $\mathring w$) to obtain:
 \begin{align}
  \int_{\StateSpace} w_t(x_t) \( \widetilde V_{t} - \Vstar_{t}\)(x_t) d x_t &\geq \int_{\StateSpace} \mathring w_t(x_t) \Biggm\{ \E_{x'|x_t,\pistar_t(x_t)}[ (\widetilde V_{t+1,\pk} - \Vstar_{t+1})(x') ] + \phi_t(x_t,\pistar_{t}(x_t))^\top \widetilde \xi_{t\pk}  \Biggm\} dx_t \\
  &\qquad\qquad + \int_{\mathcal S} \mathring w_t(x_t) \Biggm\{ \phi_t(x_t,\pistar_{t}(x_t))^\top\(\widetilde\eta_{t\pk} + \widetilde\lambda^\star_{t\pk} +  \widetilde m_{tk, \phi}^\star \) \Biggm\} dx_t
  \\
  &\qquad\qquad - 4H\epsilon \int_{\StateSpace}  w_t(x_t) d x_t
  \end{align}
Applying the definitions of $ \phi_t^\star$ and $ w_{t+1}$ and using the fact that the $ w_t$ integrate to at most 1 from (\ref{eqn:wprop}), we get that 
 \begin{align}
  \int_{\StateSpace} w_t(x_t) \( \widetilde V_{t} - \Vstar_{t}\)(x_t) d x_t &\geq 
  \int_{\StateSpace} w_t(x_t) (\widetilde V_{t+1,\pk} - \Vstar_{t+1})(x_{t+1})dx_{t+1} + (\phi_t^\star)^\top  \widetilde \xi_{t\pk} \\
  &\qquad\qquad + (\phi_t^\star)^\top \(\widetilde\eta_{t\pk} + \widetilde\lambda^\star_{t\pk} +  \widetilde m_{tk, \phi}^\star \)
  \\
  &\qquad\qquad - 4H\epsilon
  \end{align}
Then by conditioning on the good event and applying Definitions  \ref{rlsvi:defnoise_bounds} and \ref{rlsvi:defgood_event} (with our modified version of Lemma \ref{rlsvi:lemmisspec_bound}) we get that
 \begin{align}
  \int_{\StateSpace} w_t(x_t) \( \widetilde V_{t} - \Vstar_{t}\)(x_t) d x_t &\geq 
  \int_{\StateSpace} w_t(x_t) (\widetilde V_{t+1,\pk} - \Vstar_{t+1})(x_{t+1})dx_{t+1} + (\phi_t^\star)^\top \widetilde \xi_{t\pk} - \sqrt{\nu_t(\delta')} \|\phi_t^\star\|_{\Sigma_{tk}^{-1}}
  - 4H\epsilon
  \end{align}
Induction concludes the proof.
\end{proof}

\begin{lem}[Optimism]
\label{lem:Optimism}
For any episode $k$ if $ 0 < \delta <  \Phi(-1)$ and $ 0 \leq \epsilon < \frac{1}{10H}$:
	\begin{align}
	\Pro\(\widetilde V_{1}(s_{1k}) - \Vstar_1(s_{1k}) + 4 H^2\epsilon  \geq 0 \mid s_{1k}, \H_k\)\geq \Phi(-1) / 2
	\end{align}
\end{lem}
\begin{proof}
All events in this lemma are conditioned on $ s_{1k}, \mathcal{H}_k$ so that the only random variables are $ \widetilde\xi_{tk}$ for $ t \in [H]$.
Consider the probability of being optimistic at the beginning of episode $k$, and call this event $\widetilde{\mathcal{O}}_k$:
\begin{align}
\widetilde{\mathcal{O}}_k = \Big\{\(\widetilde V_{1k} - \Vstar_{1}\)(s_{1k}) \geq - 4 H^2\epsilon\Big\}.
\end{align}
For $ \epsilon < \frac{1}{10H}$, by elementary probability and using Lemma \ref{rlsvi:lemgood_prob} to bound the probability of the good event:
\begin{align}
\label{eqn:po}
\Pro(\widetilde{\mathcal{O}}_k) & = 1 - \Pro(\widetilde{\mathcal{O}}_k^c) = 1 - \Pro(\widetilde{\mathcal{O}}_k^c\cap \widetilde\G_k)  -\Pro(\widetilde{\mathcal{O}}_k^c \cap \widetilde\G^c_k) \geq 1 - \Pro(\widetilde{\mathcal{O}}_k^c \cap \widetilde\G_k)  -\Pro( \widetilde\G^c_k)\\& \geq 1 - \Pro(\widetilde{\mathcal{O}}_k^c \cap \widetilde\G_k) - \delta/8.
\end{align}
Notice that under $\G_k$, Lemma \ref{rlsvi:lemOptimisticRecursion} allows us to deduce that
\begin{align}
\label{eqn:twoinequalities}
	\(\widetilde V_{1} - \Vstar_1\)(s_{1k}) \geq \sum_{t=1}^H [(\phi^\star_t)^\top \widetilde \xi_{t\pk} - \sqrt{\nu_{k}(\delta)}\|\phi_t^\star \|_{\Sigma_{t k}^{-1}}] - 4 H^2\epsilon.
\end{align}
So, defining
\begin{align}
	\mathcal W_k \defeq \Big\{\(\widetilde V_{1} - \Vstar_1\)(s_{1k}) \geq 	\sum_{t=1}^H [(\phi^\star_t)^\top \widetilde \xi_{t\pk} - \sqrt{\nu_{k}(\delta)}\|\phi_t^\star \|_{\Sigma_{t k}^{-1}}] - 4 H^2\epsilon \Big\},
	\end{align}
we have that
\begin{align}
	\Pro\( \widetilde{\mathcal{O}}_k^c \cap \widetilde{\G}_k \) \leq \Pro\( \widetilde{\mathcal{O}}_k^c \cap \mathcal W_k \).
\end{align}

Along with equation (\ref{eqn:po}) we get:
\begin{align}
\Pro(\widetilde{\mathcal{O}}_k) & \geq 1 - \Pro\(\widetilde{\mathcal{O}}_k^c \cap \mathcal W_k \) - \delta/8.
\end{align}
Now, define the event that the random walk is positive in episode $k$:
\begin{align}
\mathcal P_k = \Big\{ \sum_{t=1}^H [(\phi^\star_t)^\top \widetilde \xi_{t\pk} - \sqrt{\nu_{k}(\delta)}\|\phi_t^\star \|_{\Sigma_{t k}^{-1}}]\geq 0 \Big\}
\end{align}
Now note that chaining the inequalities from the definitions of $ \overline{\mathcal{O}}_k$ and $ \mathcal{W}_k$ we can see that
\begin{align}
    \widetilde{\mathcal{O}}_k \cap \mathcal{W}_k \subseteq {\mathcal{P}}^c_k.
\end{align}
Thus we have
\begin{align}\label{eqn:final_p_bound}
\Pro(\widetilde{\mathcal{O}}_k) & \geq 1 - \Pro\( {\mathcal P}^c_k \) - \delta \geq \Pro\( \mathcal P_k \) - \delta / 8.
\end{align}
Recall that by the definition in the algorithm:
\begin{align}
		\widetilde \xi_{tk} \sim \mathcal N(0,H\nu_k(\delta')\Sigma^{-1}_{tk}).
	\end{align}
Now, since we have conditioned on $\H_k$ and $s_{1k}$, by Lemma \ref{rlsvi:lemOptimisticRecursion} we have that $\phi^\star_t$ is non-random and thus by properties of the normal distribution:
	\begin{align}
		(\phi^\star_t)^\top\widetilde \xi_{tk} \sim \mathcal N\(0,H\nu_k(\delta')\| \phi^\star_t \|^2_{\Sigma^{-1}_{tk}} \)
	\end{align}
	and
		\begin{align}
		\sum_{t=1}^{H}(\phi^\star_t)^\top\widetilde \xi_{tk} \sim \mathcal N\(0, H\nu_k(\delta') \sum_{t=1}^{H} \| \phi^\star_t \|^2_{\Sigma^{-1}_{tk}} \).
	\end{align}
Applying Cauchy-Schwarz we get that
\begin{align}
    \sum_{t=1}^H  \sqrt{\nu_{k}(\delta')}\|\phi_t^\star \|_{\Sigma_{t k}^{-1}} \leq \sqrt{H\nu_{k}(\delta')}\(\sum_{t=1}^H \|\phi_t^\star \|_{\Sigma_{t k}^{-1}}^2\)^{1/2},
\end{align}
which is the standard deviation of the above random variable.
Thus, we can conclude that
\begin{align}
	\Pro(\mathcal P_k) \geq \Pro\( \sum_{t=1}^{H}(\phi^\star_t)^\top\widetilde \xi_{tk}   \geq \sqrt{H\nu_{k}(\delta')}\(\sum_{t=1}^H \|\phi_t^\star \|_{\Sigma_{t k}^{-1}}^2\)^{1/2} \)\geq \Phi\(-1\).
\end{align}
Plugging this in to (\ref{eqn:final_p_bound}) and noting that $ \delta/8 < \Phi(-1)/2$ we get the result.
\end{proof}

\newpage
\section{Regret Bound}\label{sec:regret}
In this section we prove the main regret bound. This is split into two parts: one for the estimation error of each $ \overline V_{tk}$ compared to $ V_t^{\pi_k}$ and one for the pessimism of $ \overline V_{tk}$ compared to $ V_t^\star$.

\subsection{Main Theorem Statement}

\begin{theorem}[Main Result: High Probability Regret Bound for RLSVI with Approximately Linear Rewards and Low-Rank Transitions]
\label{rlsvi:thmMainResult}
Under Assumption \ref{rlsvi:assapprox_lowrank} with $\Phi(-1) > \delta > 0$ and $\lambda = 1$ and choosing $ \alpha_L, \alpha_U, \sigma^2  = H\nu_k(\delta)$ as defined in Section \ref{rlsvi:sec:good} and letting $ T=HK$, with probability at least $1-\delta$ for \textsc{opt-rlsvi} jointly for all episodes $K$:
\begin{equation}
\Regret(K) \defeq \sumk \(\Vstar_1 -  V^{\pi_{\pk}}_1\) (\sok) =\widetilde O\(\sqrt{\gamma_K(\delta)}\sqrt{dHT}  + \frac{H^2d}{\alpha_L^2} + \epsilon HT \) .
\end{equation}
\end{theorem}
\begin{proof}
We have the following decomposition:
\begin{align}
\Regret(K) \defeq \sumk \(\Vstar_1 -  V^{\pi_{\pk}}_1\) (\sok) & = \sumk \(\Vstar_1 - \vo_{1k}\)(s_{1k}) + \sumk \(\vo_{1k} - V^{\pi_{\pk}}_1\) (\sok).
\end{align}
Taking a union bound over the results of Lemma \ref{rlsvi:lemestimation} and Lemma \ref{rlsvi:lemPessimism} yields the result.
\end{proof}

\begin{corollary}[High Probability Regret Bound for RLSVI with Approximately Linear Rewards and Low-Rank Transitions] Under Assumption \ref{rlsvi:assapprox_lowrank} and if additionally $ L_\phi = \widetilde O(1)$, and $L_\psi, L_r = \widetilde O(d)$, then with probability at least $1-\delta$ for \textsc{opt-rlsvi} it holds that:
\begin{equation}
\Regret(K) \defeq \sumk \(\Vstar_1 -  V^{\pi_{\pk}}_1\) (\sok) = \widetilde O\(H^2 d^2 \sqrt{T} + H^5 d^4  + \epsilon dH (1 + \epsilon dH^2) T  \).
\end{equation}
\end{corollary}
\begin{proof}
Recall from Definition \ref{rlsvi:defnoise_bounds} we have that
\begin{align}
    \sqrt{\gamma_K(\delta)} = \widetilde O((Hd)^{3/2} + \sqrt{Hd \lambda} L_\phi(3HL_\psi + L_r) + \epsilon \sqrt{dHT}) = \widetilde O((Hd)^{3/2} + \epsilon \sqrt{dHT})
\end{align}
And combining with Definition \ref{rlsvi:defdefault} we have
\begin{align}
    \frac{1}{\alpha_L^2} = \widetilde O((Hd)^{3} + \epsilon^2 dHT)
\end{align}
Plugging these values into Theorem \ref{rlsvi:thmMainResult} we get with probability at least $ 1- \delta $ that
\begin{align}
    \Regret(K) &= \widetilde O\(((Hd)^{3/2} + \epsilon \sqrt{dHT})\sqrt{dHT}  + (H^2d)((Hd)^{3} + \epsilon^2 dHT) + \epsilon HT \)\\
    &= \widetilde O\(H^2 d^2 \sqrt{T} + \epsilon dHT + H^5 d^4 + \epsilon^2 d^2 H^3 T + \epsilon HT \)\\
    &= \widetilde O\(H^2 d^2 \sqrt{T} + H^5 d^4  + \epsilon dH(1 + \epsilon dH^2) T  \).
\end{align}
\end{proof}

\subsection{Bounding the Estimation Error}

\begin{lemma}[Bound on Estimation]
\label{rlsvi:lemestimation}
It holds with probability at least $ 1- \delta / 2 $ that:
\begin{align}
	\sumk \(\vo_{1\pk} - V^{\pi_{\pk}}_1\)(\sok) = \widetilde O\((\sqrt{\nu_K(\delta')} + \sqrt{\gamma_K(\delta')})H \sqrt{d}\sqrt{K}  + H^2K \epsilon + \frac{H^2 d}{\alpha_L^2} \)
\end{align}
\end{lemma}
\begin{proof}
The proof proceeds by induction over $ t \in [H]$ followed by some algebra to get the bound. Denote by $ G_k $ the event that $ \overline {\mathcal{G}}_\ell$ (see Def.~\ref{rlsvi:defgood_event}) holds for all $ \ell \leq k$, so that $ G_k $ is measurable with respect to $ \overline{\mathcal{H}}_k$.

Consider a generic timestep $t$: we split into two cases. Either 1) we have $ \|\phi_{tk}\|_{\Sigma_{tk}^{-1}} \leq \alpha_L$ which we will call $ \mathcal{S}_{tk}$ or 2) we have $ \|\phi_{tk}\|_{\Sigma_{tk}^{-1}} > \alpha_L$ which we will call $ \mathcal{S}_{tk}^c$. Under $ \mathcal{S}_{tk}$ the Q function is linear (see Eq.~\ref{eqn:defQthetasigma} or Def.~\ref{rlsvi:defQdef}), and under $ \mathcal{S}_{tk}^c$ we can upper bound the value function difference by $ H $ in the worst case under $ G_k$. Thus we have
\begin{align}
    \(\vo_{t\pk} - V^{\pi_{\pk}}_t\)(\stk)\1\{G_k\} & =  \1\{G_k\}\(\(\vo_{t\pk} - V^{\pi_{\pk}}_t\)(\stk)\1\{\mathcal S_{tk}\} +  \(\vo_{t\pk} - V^{\pi_{\pk}}_t\)(\stk)\1\{\mathcal S^c_{tk}\}\) \\
    &= \1\{G_k\}\(\(\vo_{t\pk}(s_{tk}) - Q^{\pi_{\pk}}_t(s_{tk}, a_{tk})\)\1\{\mathcal S_{tk}\} +  \(\vo_{t\pk} - V^{\pi_{\pk}}_t\)(\stk)\1\{\mathcal S^c_{tk}\}\)\\
    & \leq   \1\{G_k\}\(\underbrace{\(\phi_{tk}^\top\thetabar_{t\pk} - Q^{\pi_{\pk}}_t(s_{tk}, a_{tk})\) \1\{\mathcal S_{tk}\}}_{S} + \underbrace{H\1\{\mathcal S^c_{tk}\}}_{S^c}\).
\end{align}
We focus on the first term, term $S$.
Applying Lemma \ref{rlsvi:lemQ_decomp} we have
\begin{align}
    \phi_{tk}^\top\thetabar_{t\pk} - Q_{t}^{\pi_{\pk}}(\stk,a_{tk}) = \E_{s'|s,a}[\( \overline V_{t+1,\pk} - V^{\pi_k}_{t+1} \)(s')]  + \phi_{tk}^\top (\overline \eta_{tk} + \overline\xi_{tk} + \overline \lambda_{tk}^{\pi_k}) + \overline m_{tk}^{\pi_k}(s,a).
\end{align}
And under $ \overline{\mathcal{G}}_k$ we can bound this by
\begin{align}
    \phi_{tk}^\top\thetabar_{t\pk} - Q_{t}^{\pi_{\pk}}(\stk,a_{tk}) \leq \E_{s'|s_{tk},a_{tk}}[\( \overline V_{t+1,\pk} - V^{\pi_k}_{t+1} \)(s')]  + (\sqrt{\nu_k(\delta')} + \sqrt{\gamma_k(\delta')})\|\phi_{tk}\|_{\Sigma_{tk}^{-1}} + 4 H \epsilon.
\end{align}
Then we can define
\begin{align}
    \dot\zeta_{tk} \defeq \1\{G_k\}\1\{\mathcal S_{tk}\}\(\E_{s'|s_{tk},a_{tk}}[\( \overline V_{t+1,\pk} - V^{\pi_k}_{t+1} \)(s')] - \( \overline V_{t+1,\pk} - V^{\pi_k}_{t+1} \)(s_{t+1,k})\).
\end{align}
Note that due to the indicator of $ G_k$ we have that each $ |\dot\zeta_{tk}| \leq 2H$ a.s. and $\mathbb{E}[\dot\zeta_{tk} | \overline{\mathcal H}_k \cup \mathcal{H}_{tk}] =0 $. Then $(\dot\zeta_{tk}, \overline{\mathcal H}_k \cup \mathcal{H}_{tk})_{t,k}$ is an MDS. So, applying Azuma-Hoeffding we have with probability at least $ 1-\delta/4$ that $ \sum_{k=1}^K \sum_{t=1}^H \dot \zeta_{tk} = \widetilde O (H\sqrt{T})$.

With this definition,
\begin{align}
    \1\{G_k\}S \leq \1\{G_k\}\1\{\mathcal S_{tk}\}\(\( \overline V_{t+1,\pk} - V^{\pi_k}_{t+1} \)(s_{t+1,k}) + (\sqrt{\nu_k(\delta')} + \sqrt{\gamma_k(\delta')})\|\phi_{tk}\|_{\Sigma_{tk}^{-1}} + 4 H \epsilon\) + \dot\zeta_{tk}.
\end{align}
Combining it all we have
\begin{align}
    \1\{G_k\} \(\vo_{t\pk} - V^{\pi_{\pk}}_t\)(\stk) &\leq \1\{G_k\}\bigg[\( \overline V_{t+1,\pk} - V^{\pi_k}_{t+1} \)(s_{t+1,k}) \1\{\mathcal S_{tk}\} \\& + \(\Big(\sqrt{\nu_k(\delta')} + \sqrt{\gamma_k(\delta')}\Big)\|\phi_{tk}\|_{\Sigma_{tk}^{-1}} + 4 H \epsilon  \)\1\{\mathcal S_{tk}\} +  H\1\{\mathcal S^c_{tk}\}\bigg] + \dot\zeta_{tk}.
\end{align}
And induction gives us
\begin{align}
    \1\{G_k\}\(\vo_{1\pk} - V^{\pi_{\pk}}_1\)(\stk) &\leq \1\{G_k\}\sum_{t=1}^H \bigg[ \(\Big(\sqrt{\nu_k(\delta')} + \sqrt{\gamma_k(\delta')}\Big)\|\phi_{tk}\|_{\Sigma_{tk}^{-1}} + 4 H \epsilon \)\( \Pi_{\tau=1}^t \1\{\mathcal S_{\tau k}\} \) \\
    & + H\( \Pi_{\tau=1}^{t-1} \1\{\mathcal S_{\tau k}\}\) \1\{\mathcal S^c_{tk}\}  \bigg] + \sum_{t=1}^H \dot\zeta_{tk}
\end{align}
Now we can sum over $ k $ to attain a bound on the estimation error term of the regret. We will split this in three terms: when all $ \mathcal{S}_{tk}$ occur, when some $ \mathcal{S}_{tk}^c$ occurs, and the martingale difference terms. We can bound the dominant term by exchanging order of summation, pulling out constants, applying Cauchy-Schwarz, and finally applying Lemma \ref{rlsvi:lemfeature_sum_all} to get
\begin{align}
    & \sum_{k=1}^K \1\{G_k\}\sum_{t=1}^H \((\sqrt{\nu_k(\delta')} + \sqrt{\gamma_k(\delta')})\|\phi_{tk}\|_{\Sigma_{tk}^{-1}} + 4 H \epsilon \)\( \Pi_{\tau=1}^t \1\{\mathcal S_{\tau k}\} \)\\
    &\qquad\leq  (\sqrt{\nu_K(\delta')} + \sqrt{\gamma_K(\delta')})\sum_{t=1}^H \sum_{k=1}^K \1\{G_k\} \|\phi_{tk}\|_{\Sigma_{tk}^{-1}}\( \Pi_{\tau=1}^t \1\{\mathcal S_{\tau k}\} \)  + 4 H^2K \epsilon \\
    &\qquad\leq (\sqrt{\nu_K(\delta')} + \sqrt{\gamma_K(\delta')}) \sum_{t=1}^H \sqrt{K}\(\sum_{k=1}^K \|\phi_{tk}\|^2_{\Sigma_{tk}^{-1}} \( \Pi_{\tau=1}^t \1\{\mathcal S_{\tau k}\} \)\)^{1/2}  + 4 H^2K \epsilon\\
    &\qquad\leq (\sqrt{\nu_K(\delta')} + \sqrt{\gamma_K(\delta')}) \sum_{t=1}^H \sqrt{K}\(\sum_{k=1}^K \min\{1, \|\phi_{tk}\|^2_{\Sigma_{tk}^{-1}} \}\)^{1/2} + 4 H^2K \epsilon\\
    &\qquad \leq (\sqrt{\nu_K(\delta')} + \sqrt{\gamma_K(\delta')})H \sqrt{K} \tilde O(\sqrt{d})  + 4 H^2K \epsilon
\end{align}
To get the conclusion we need to show that the desired bound holds with high probability. Note that if $ \epsilon > \frac{1}{10H}$ the bound we are trying to prove is trivially true since it is larger than $ T$. So, assuming $ \epsilon < \frac{1}{10H}$ and applying Lemma \ref{rlsvi:lemgood_prob} we get that $ \bigcap_{k \in [K]} G_k = \bigcap_{k \in [K]} \overline {\mathcal{G}}_k$ occurs with probability at least $ 1 - \delta /8$. Taking a union bound we see that with probability at least $ 1 - \delta/2$ both the $ G_k $ and the bound on the sum of the $ \dot \zeta_{tk}$ hold. Adding the lower order term bound from Lemma \ref{rlsvi:lemWarmup} gives the desired result.
\end{proof}

\subsection{Bounding the Pessimism}

\begin{lemma}[Bound on Pessimism]
\label{rlsvi:lemPessimism}
For any $\Phi(-1) > \delta > 0 $ it holds with probability at least $ 1- \delta / 2 $ that:
\begin{align}
	\sumk \(\Vstar_{1\pk} - \vo_{1\pk} \)(\sok) = \widetilde O\((\sqrt{\nu_K(\delta')} + \sqrt{\gamma_K(\delta')})H \sqrt{d}\sqrt{K}  + H^2K \epsilon + \frac{H^2 d}{\alpha_L^2} \).
\end{align}
\end{lemma}
\begin{proof}
In this section we bound the pessimism term by connecting it to the probability of the algorithm being optimistic and the concentration terms. Essentially, we construct an upper bound on $ V^\star$ and a lower bound on $ \overline{V}_{1k}$ and show that they cannot be too different from each other.

As in the previous proof, we will use indicator functions of a good event. But, in this proof we will not just have the $ \overline{\xi}$ pseudonoise variables but also $ \widetilde\xi$ and $ \underline{\xi}$ (defined later in the proof).  These variables have good events $ \widetilde {\mathcal{G}}_k, \underline{\mathcal{G}}_k$ defined per episode analogous to $ \overline {\mathcal{G}}_k$ (see Def.~\ref{rlsvi:defgood_event}). Accordingly we will now denote by $ G_k $ the event that $ \overline {\mathcal{G}}_\ell \cap \widetilde {\mathcal{G}}_\ell \cap \underline{\mathcal{G}}_\ell$ holds for all $ \ell \leq k$, so that $ G_k $ is measurable with respect to $ \overline{\mathcal{H}}_k$. Note that by Lemma \ref{rlsvi:lemgood_prob} and a union bound over the three pseudonoises we have that $ \bigcap_{k \in [K]} G_k $ occurs with probability at least $ 1 - 3\delta/8$.

First we construct the lower bound.
Let the $\xi_{tk}$'s be vectors in $\R^d$ for $t = 1,\dots,H$, and let $V^\xi_{tk}$ be the value function obtained by running the Least Square Value Iteration procedure in Algorithm \ref{alg:RLSVI} backward with the non-random $\xi_{tk}$ (see definition below) in place of $\overline \xi_{tk}$.
Consider the following minimization program:
\begin{equation}\label{eq:minproblem}
\begin{aligned}
& \min_{\{\xi_{tk} \}_{t=1,\dots,H}}  V^\xi_{1k}(s_{1k}) \\
& \|\xi_{tk}\|_{\Sigma_{tk}}  \leq \sqrt{\gamma_k(\delta')}, \quad \forall t \in [H]
\end{aligned}
\end{equation}
Notice that the constraint condition on the $\xi$ variables is equivalent to the one on the $\overline \xi$ in the definition of $\mathcal{G}_{tk}^{\overline\xi}$ in Definition \ref{rlsvi:defgood_event}, but with $ \xi_{tk}$ replacing the $ \overline{\xi}_{tk}$.
We denote with $\{ \underline \xi_{tk} \}_{t=1,\dots,H}$ a minimizer of the above expression and with $\underline V_{1k}(s_{1k})$ the minimum of the optimization program (the minimum exists because $V^\xi_{1k}(s_{1k})$ is a continuous function of the $\xi$  which are defined on a compact set). Importantly, under $ \overline{\mathcal{G}}_{k}$ we get that
\begin{align}\label{eqn:pessimism_lower}
    \underline{V}_{1k}(s_{1k}) \leq \overline{V}_{1k}(s_{1k})
\end{align}
because $\{ \overline \xi_{tk} \}_{t=1,\dots,H}$ is a feasible solution of the optimization and $ V^{\overline\xi}_{tk}(s_{1k}) = \overline V_{tk}(s_{1k})$.

Next, we want to get an upper bound. Consider drawing an independent and identically distributed copy $\widetilde \xi_{tk}$ of the $\overline \xi_{tk}$'s and run the least square procedure backward to get a new value function $\widetilde V_{tk}$ (for $t \in [H]$) and action-value function $\widetilde Q_{tk}$. Define as $\widetilde{\mathcal O}_k$ the event that $\widetilde V_{1k}(s_{1k})$ is optimistic in the $k$-th episode. Applying Lemma \ref{lem:Optimism} with $\Phi(-1) > \delta > 0$ and $\epsilon \leq \frac{1}{10H}$,
\begin{align}\label{eqn:tilde_prob}
\Pro\( \widetilde{\mathcal{O}}_k \) = \Pro \( \{ \widetilde V_{1k}(s_{1k}) \geq \Vstar_{1k}(s_{1k}) - 4H^2\epsilon\} \) \geq \Phi(-1)/2.
\end{align}
Next using this definition of optimism we can write:
\begin{align}
\label{eqn:someEqh}
\(\Vstar_{1k}-\overline V_{1k}\)(s_{1k}) \1\{\overline{\G}_k\}
& \leq \E_{\widetilde\xi|\widetilde{\mathcal O}_k} \bigg[ \( \widetilde V_{1k}-\overline V_{1k}\)(s_{1k})\bigg]\1\{\overline{\G}_k\} +4H^2\epsilon \\
& \leq \E_{\widetilde\xi|\widetilde{\mathcal O}_k} \bigg[\(\widetilde V_{1k}-\underline V_{1k}\)(s_{1k})\bigg]\1\{\overline{\G}_k\} +4H^2\epsilon.
\end{align}
where the expectations are over the $\widetilde \xi$'s, conditioned on the event $\widetilde{\mathcal{O}}_k$. The second bound follows from Equation (\ref{eqn:pessimism_lower}).

At this point we can use the law of total expectation under $ \widetilde{\mathcal{G}}_k$:
\begin{align}
\E_{\widetilde\xi}\bigg[ \(\widetilde V_{1k} - \underline V_{1k} \)(s_{1k}) \bigg] &  = \E_{\widetilde\xi|\widetilde{\mathcal O}_k}\bigg[ \(\widetilde V_{1k} - \underline V_{1k} \)(s_{1k})\bigg] \Pro(\widetilde{\mathcal O}_k) +  \underbrace{\E_{\widetilde\xi|\widetilde{\mathcal O}_k^c} \bigg[\(\widetilde V_{1k} - \underline V_{1k}  \)(s_{1k})\bigg]}_{\geq 0} \Pro(\widetilde{\mathcal O}_k^c) \\
 & \geq \E_{\widetilde\xi|\widetilde{\mathcal O}_k}\bigg[ \(\widetilde V_{1k} - \underline V_{1k}\)(s_{1k})\bigg] \Pro(\widetilde{\mathcal O}_k).
\end{align}
The lower bound again follows because $\{ \widetilde \xi_{tk} \}_{t=1,\dots,H}$ is a feasible solution of~\eqref{eq:minproblem}, so the neglected term is positive. Chaining the above with (\ref{eqn:tilde_prob}) and (\ref{eqn:someEqh}) and using the definition of $ G_k $ (i.e., $G_k \implies \overline{\mathcal{G}}_k$):
\begin{align}
\1\{G_k\}\(\Vstar_{1k}-\overline V_{1k}\)(s_{1k}) & \leq \1\{G_k\}\frac{2}{\Phi(-1)}\E_{\widetilde\xi}\bigg[ \(\widetilde V_{1k} - \underline V_{1k} \)(s_{1k})\bigg] + 4H^2\epsilon \\
& = \1\{G_k\}\frac{2}{\Phi(-1)} \(\overline V_{1k} - \underline V_{1k} \)(s_{1k}) + \ddot \zeta_k +4H^2\epsilon \\
& = \1\{G_k\}\frac{2}{\Phi(-1)} \(\overline V_{1k} -  V^{\pi_k}_{1} +  V^{\pi_k}_{1} - \underline V_{1k} \)(s_{1k}) + \ddot \zeta_k +4H^2\epsilon \label{eqn:pessimism_final_ref}
\end{align}
where we define
\begin{align}
    \ddot\zeta_{k} \defeq  \1\{G_k\}\frac{2}{\Phi(-1)}\(\E_{\widetilde\xi} \bigg[\widetilde V_{1k} (s_{1k}) \bigg]-  \overline V_{1k} (s_{1k})\)
\end{align}
and note that since the $ \overline{\xi}_{tk} $ and $ \widetilde \xi_{tk}$ are iid, so are $ \widetilde V_{1k}$ and $ \overline{V}_{1k}$.
Then $(\ddot\zeta_{k}, {\mathcal H}_{k-1})_{k}$ is an MDS 
and due to the indicator function each term is bounded in absolute value by 2H. So, applying Azuma-Hoeffding we have with probability at least $ 1-\delta/16$ that $ \sum_{k=1}^K \ddot \zeta_{tk} = \widetilde O (H\sqrt{K})$.

Now we decompose
\begin{align}
    \1\{G_k\}\(\overline V_{1k} -  V^{\pi_k}_{1} +  V^{\pi_k}_{1} - \underline V_{1k} \)(s_{1k}) = \1\{G_k\}\(\overline V_{1k} -  V^{ \pi_k}_{1}\)(s_{1k})  +  \1\{G_k\}\(V^{\pi_k}_{1} - \underline V_{1k} \)(s_{1k})
\end{align}
The first term is the estimation error that we bounded in Lemma \ref{rlsvi:lemestimation}.

For the second term, we can derive the same bound, but require a slightly modified proof. As before, we set up the recursion by considering a generic timestep $ t $ and splitting into cases, bounding the difference by $H $ on $ \mathcal{S}_{tk}^c$ (see definition in Lem.~\ref{rlsvi:lemestimation}):
\begin{align}
    \1\{G_k\}\( V^{ \pi_k}_{t} - \underline V_{tk} \)(s_{tk}) \leq \1\{G_k\}\(\( V^{ \pi_{\pk}}_t - \underline{V}_{tk}\)(\stk)\1\{\mathcal S_{tk}\} +  H\1\{\mathcal S^c_{tk}\}\)
\end{align}
Now consider the term where $\Smallcondition{tk}{\phi_{tk}}$ holds. First note that since $ a_{tk} $ is the action that maximizes $ \overline{Q}_{tk}$,
 \begin{align}
 \label{eqn:129}
 \( V^{ \pi_k}_{t} - \underline V_{tk} \)(s_{tk}) & = Q^{ \pi_k}_{t}(s_{tk}, a_{tk}) - \underline V_{tk}(s_{tk})  \leq \(Q^{\pi_k}_{t} - \underline Q_{tk} \)(s_{tk},a_{tk}) = Q^{\pi_k}_{t}(s_{tk},a_{tk}) - \phi_{tk}^\top \underline \theta_{tk} .
\end{align}
Applying Lemma\footnote{Note that Lemma \ref{rlsvi:lemQ_decomp} is derived for $\overline \theta_{tk}$, but we can derive an equivalent expression for $\underline \theta_{tk}$} \ref{rlsvi:lemQ_decomp} we see that this is
\begin{align}
    Q^{\pi_k}_{t}(s_{tk},a_{tk}) - \phi_{tk}^\top \underline \theta_{tk}  = - \E_{s'|s_{tk},a_{tk}}[\( \underline V_{t+1,\pk} - V^{\pi_k}_{t+1} \)(s')]  - \phi_{tk}^\top (\underline \eta_{tk} + \underline\xi_{tk} + \underline \lambda_{tk}^{\pi_k}) - \underline m_{tk}^{\pi_k}(s_{tk},a_{tk}).
\end{align}
And we can define
\begin{align}
    \dddot\zeta_{tk} \defeq \1\{G_k\}\(- \E_{s'|s_{tk},a_{tk}}[\( \underline V_{t+1,\pk} - V^{\pi_k}_{t+1} \)(s')]  + \( \underline V_{t+1,\pk} - V^{ \pi_k}_{t+1} \)(s_{t+1,k})\)
\end{align}
Then $(\dddot\zeta_{tk}, \overline{\mathcal H}_k \cup \mathcal{H}_{tk})_{t,k}$ is an MDS and due to the indicator function each term is bounded in absolute value by 2H. So, applying Azuma-Hoeffding we have with probability at least $ 1-\delta/16$ that $ \sum_{k=1}^K \sum_{t=1}^H \ddot \zeta_{tk} = \widetilde O (H\sqrt{T})$

So that, as in Lemma \ref{rlsvi:lemestimation}, induction gives us
\begin{align}
    \1\{{G}_k\}\(V^{\pi_k}_{1} - \underline V_{1k} \)(s_{1k})  &\leq \1\{{G}_k\}\sum_{t=1}^H \bigg[ \((\sqrt{\nu_k(\delta')} + \sqrt{\gamma_k(\delta')})\|\phi_{tk}\|_{\Sigma_{tk}^{-1}} + 4 H \epsilon  \)\( \Pi_{\tau=1}^t \1\{\mathcal S_{\tau k}\} \) \\
    &\qquad\qquad + H\( \Pi_{\tau=1}^{t-1} \1\{\mathcal S_{\tau k}\}\) \1\{\mathcal S^c_{tk}\}  \bigg] + \sum_{t=1}^H \dddot\zeta_{tk}
\end{align}
Summing over $ k$, this can be bounded as in Lemma \ref{rlsvi:lemestimation}. To conclude, summing the bound from (\ref{eqn:pessimism_final_ref}) over $ k $ and applying the same arguments as Lemma \ref{rlsvi:lemestimation} to both value function differences gives us that
\begin{align}
    \sumk \1\{{G}_k\}\(\Vstar_{1\pk} - \vo_{1\pk} \)(\sok) & \leq \frac{4}{\Phi(-1)}  \widetilde O\((\sqrt{\nu_K(\delta')} + \sqrt{\gamma_K(\delta')})H \sqrt{d}\sqrt{K}  + H^2K \epsilon + \frac{H^2 d}{\alpha_L^2} \) \\
    & + \tilde O(H \sqrt{K}) + 4HT\epsilon
\end{align}
so that consolidating terms gives us the desired bound. Notice that if $\epsilon \geq \frac{1}{10H}$ the result trivially holds.

To conclude we just need to take a union bound over the two applications of Azuma-Hoeffding and the intersection of the $ G_k $ we get the result with probability $ 1 - \delta/2$ as desired.
\end{proof}

\subsection{Bounding the Warmup}

\begin{lemma}[Warmup Bound]
\label{rlsvi:lemWarmup}
	\begin{align}
	 \sumk \sum^H_{t=1}H \1\Big\{ \mathcal{S}^c_{tk} \Big\} \defeq \sumk \sum^H_{t=1}H \1\Big\{ \| \phi_{tk} \|_{\Sigma^{-1}_{tk}} > \alpha_L \Big\} = \widetilde O\(\frac{H^2d}{\alpha_L^2}\).
	\end{align}

\end{lemma}
\begin{proof}
\begin{align}
\sumk \sum^H_{t=1}H \1\Big\{ \| \phi_{tk} \|_{\Sigma^{-1}_{tk}} > \alpha_L \Big\} & = H\sumk\sum^H_{t=1} \1\Bigg\{ \frac{\| \phi_{tk} \|_{\Sigma^{-1}_{tk}}}{\alpha_L} > 1 \Bigg\} \\
& = H\sumk\sum^H_{t=1} \1\Bigg\{ \frac{\| \phi_{tk} \|^2_{\Sigma^{-1}_{tk}}}{\alpha^2_L} > 1 \Bigg\} \\
& \leq H\sumk\sum^H_{t=1} \min \bigg\{1, \frac{\| \phi_{tk} \|^2_{\Sigma^{-1}_{tk}}}{\alpha^2_L} \bigg\}  \\
& \stackrel{(a)}{\leq} \frac{H}{\alpha_L^2}\sum^H_{t=1}\sumk  \min\{1, \| \phi_{tk} \|^2_{\Sigma^{-1}_{tk}} \}  \\
& \stackrel{(b)}{\leq} \frac{H^2}{\alpha_L^2} \widetilde O(d) = \widetilde O\(\frac{H^2d}{\alpha_L^2}\)
\end{align}
Where (a) holds since $ 1/\alpha^2 > 1$ by the following reasoning. Let $ x > 1$ and consider two cases: if $ y < 1/x $ then $ \min\{1, x y\} = xy = x \min\{1, y\}$ and if $ y \geq 1/x$ then $ \min\{1, x y\} = 1 \leq x \leq x \min\{1, y\} $.
Finally, (b) is due to Lemma \ref{rlsvi:lemfeature_sum_all}.
\end{proof}

\newpage
\section{Computational Complexity}\label{sec:complexity}

Now we take a look at the computational complexity of the algorithm.

\begin{proposition}[Computational Complexity of \textsc{opt-rlsvi} in finite action spaces]
\label{rlsvi:propComputationalComplexity}
Let $A$ be the number of actions available at every timestep. Then \textsc{opt-rlsvi} can be implemented in space $O(d^2H+dAHK)$ and time $O(d^2AHK^2)$.
\end{proposition}
\begin{proof}
In terms of computational complexity, a naive implementation of \textsc{opt-rlsvi} requires $O(d^2)$ elementary operations to compute $\| \phi_{t+1,i} \|_{\Sigma^{-1}_{t+1,k}}$ to assess which decision rule to use in definition \ref{rlsvi:defQdef}.
This must be done for all next-state action-value functions at the experienced successors states. If the action space is finite with cardinality $A$ then the maximization over action to compute the value function $\overline V_{t+1,k}(s_{t+1,i})$ at the next timestep for the $k$ experienced successor states $s_{t+1,1},\dots,s_{t+1,k}$ would take $O(d^2AK)$ total work per timestep. A further $O(d^3)$ is needed to compute the inverse of $\Sigma_{tk}$ to solve the least square system of equation, but this can be brought down to $O(d^2)$ using the usual Sherman-Morrison rank one update formula. All this must be done at every timestep of the least-square value iteration procedure, which must run every episode, giving a final runtime $O(d^2AHK^2)$.

As for the memory, one can store the $K$ features $\phi_t(s_{tk},a)$
 for all $A$ actions, timestep $H$ and episode $K$ using $O(dAHK)$ memory, in addition to the inverse of the $\Sigma_{tk}$ matrices ($O(d^2H)$ space) and the scalar rewards ($O(KH)$ space).
\end{proof}

\newpage

\section{Technical Lemmas}\label{sec:technical}

\begin{lemma}[Self-normalized process ] \citep{Abbasi11}\label{rlsvi:lemself_normalized}
Let $ \{x_i\}_{i=1}^\infty$ be a real valued stochastic process sequence over the filtration $ \{\mathcal{F}_i\}_{i=1}^\infty$. Let $ x_i $  be conditionally $ B$-subgaussian given $ \mathcal{F}_{i-1}$. Let $ \{\phi_i\}_{i=1}^\infty $ with $ \phi_i \in \mathcal{F}_{i-1}$ be a stochastic process in $ \R^d$ with each $ \|\phi_i\|\leq L_\phi$. Define $ \Sigma_i = \lambda I + \sum_{j=1}^{i-1}\phi_i \phi_i^\top$.
Then for any $ \delta>0$ and all $ i \geq 0$, with probability at least $ 1-\delta$
\begin{align}
    \bigg\| \sum_{i=1}^{k-1} \phi_{i} x_{i} \bigg\|_{\Sigma_{k}^{-1}}^2 \leq 2B^2 \log \bigg(\frac{\det (\Sigma_i)^{1/2} \det(\lambda I)^{-1/2}}{\delta} \bigg) \leq 2B^2\bigg(d\log\(\frac{ \lambda + k L_\phi^2}{\lambda}\) + \log(1/\delta)\bigg)
\end{align}
\end{lemma}

\begin{lemma}[Sum of features] \citep[Lemma 11]{Abbasi11}\label{rlsvi:lemfeature_sum_all}  Using the notation defined above,
\begin{align}
     \sum_{i=1}^k \min\{1, \|\phi_{i}\|^2_{\Sigma_{i}^{-1}}\} \leq 2d\log\(\frac{\lambda + kL_\phi^2}{\lambda}\)
\end{align}
\end{lemma}

\begin{lemma}[Sum of features in final norm]\label{rlsvi:lemfeature_sum_final} \citep[Lemma D.1]{jin2019provably}
\begin{align}
    \sum_{i=1}^{k-1} \|\phi_{i} \|^2_{\Sigma_{k}^{-1}} \leq d
\end{align}
\end{lemma}

\begin{lemma}[Gaussian concentration]\label{rlsvi:lemgaussian_concentration} \citep[Appendix A]{abeille2017linear}
Let $ \overline{\xi}_{tk} \sim \mathcal{N}(0, H\nu_k(\delta) \Sigma_{tk}^{-1}) $. For any $ \delta > 0$, with probability $ 1- \delta$
\begin{align}
    \|\overline\xi_{tk}\|_{\Sigma_{tk}} \leq c \sqrt{H d \nu_k(\delta) \log(d/\delta)} 
\end{align}
for some absolute constant $ c$.
\end{lemma}

\begin{lemma}[Covering numbers]\label{rlsvi:lemcovering_number}\citep[Section 4]{pollard1990empirical}
A euclidean ball of radius $ B$ in $ \R^d$ has $ \varepsilon$-covering number at most $ ( 3B/\varepsilon)^d$.
\end{lemma}

\begin{lemma}[Simplifying the log term]\label{rlsvi:lemlog_term} With $ \lambda \geq 1$, we can choose $ c_1 $ so that
    \begin{align}
       \sqrt{\beta_k(\delta)} & \geq  8Hd \bigg(\log\(\frac{k L_\phi^2 + \lambda}{\lambda}\) \\
       & + \log \bigg(3 (2H\sqrt{kd/\lambda} + \sqrt{\gamma_k(\delta)/\lambda} + 1/\lambda)/\(\frac{\alpha_U - \alpha_L}{8kL_\phi^2 H^2}\)^2\bigg) + \log(1 /\delta)\bigg)^{1/2}
    \end{align}
\end{lemma}
\begin{proof}
Recall that
\begin{align}
    \sqrt{\beta_k(\delta)} &\defeq c_1 Hd \sqrt{\log \(\frac{Hdk \max(1,L_\phi) \max(1,L_\psi) \max(1,L_r) \lambda}{\delta}\)}
\end{align}
Using $\lambda \geq $ and expanding the definitions of terms on the RHS of the statement we can bound it by
\begin{align}
    &\leq  8 Hd \bigg( \log \(\frac{(k L_\phi^2 + \lambda)3( 2H\sqrt{kd} + \sqrt{\gamma_k(\delta)} + 1)64k^2 L_\phi^4 H^4}{\delta (\alpha_U - \alpha_L)^2}\) \bigg)^{1/2}\\
    &\leq 8  Hd \bigg( \log \(\frac{(k L_\phi^2 + \lambda)(H\sqrt{kd} + \sqrt{\gamma_k(\delta)})64k^2 L_\phi^4 H^2(\gamma_k(\delta))}{\delta \lambda}\) \bigg)^{1/2}
    \end{align}
    \scriptsize
    \begin{align}
    \leq 8 Hd \bigg( \log \(\frac{(k L_\phi^2 + \lambda)(H\sqrt{kd})k^2 L_\phi^4 H^2(c_2^2 dH (\sqrt{\beta_k(\delta)} + \sqrt{\lambda} L_\phi ( 3HL_\psi + L_r ) + 4 \epsilon H \sqrt{dk})^2 \log(d/\delta))^{3/2}}{\delta}\) \bigg)^{1/2}
\end{align}
\normalsize
Bounding the $ \sqrt{\beta_k(\delta)}$ by $ c_1 Hd (Hdk\max(1,L_\phi) \max(1,L_\psi) \max(1,L_r) \lambda)/\delta$ this gives us a large polynomial in $ k,H, d, \lambda, \max(1,L_\phi), \max(1,L_\psi), \max(1,L_r), 1/\delta$. We bound this by $ c(kH d \lambda \max(1,L_\phi) \max(1,L_\psi) \max(1,L_r)/\delta)^{c'}$ for some $ c, c'$, and taking the log to move the exponent into the constant gives the existence of some $ c_1$ to define $ \beta_k(\delta)$.    
\end{proof}

\end{document}